\documentclass{article}

\PassOptionsToPackage{numbers,sort&compress}{natbib}

\usepackage[final,main]{neurips_2026}

\usepackage[utf8]{inputenc}
\usepackage[T1]{fontenc}
\usepackage{microtype}
\usepackage{xcolor}

\usepackage{amsmath,amsthm,amssymb,amsfonts}
\usepackage{mathtools}
\usepackage{thmtools}

\usepackage{graphicx}
\usepackage{float}
\usepackage{tikz}
\usetikzlibrary{arrows.meta,positioning,shapes,calc,decorations.pathreplacing}

\usepackage{booktabs}
\usepackage{array}

\usepackage{pgfplots}
\pgfplotsset{compat=1.17}

\usepackage{url}
\usepackage[hidelinks]{hyperref}

\theoremstyle{plain}
\newtheorem{theorem}{Theorem}[section]
\newtheorem{lemma}[theorem]{Lemma}
\newtheorem{proposition}[theorem]{Proposition}
\newtheorem{corollary}[theorem]{Corollary}

\theoremstyle{definition}
\newtheorem{definition}[theorem]{Definition}

\newtheorem{remark}[theorem]{Remark}

\newcommand{\E}{\mathbb{E}}
\newcommand{\Prob}{\mathbb{P}}
\newcommand{\calM}{\mathcal{M}}
\newcommand{\calF}{\mathcal{F}}
\newcommand{\calA}{\mathcal{A}}
\newcommand{\calG}{\mathcal{G}}
\newcommand{\calE}{\mathcal{E}}
\newcommand{\Obs}{\mathsf{Obs}}
\newcommand{\Int}{\mathsf{Int}}
\newcommand{\CF}{\mathsf{CF}}
\newcommand{\DL}{\mathrm{DL}}
\newcommand{\pa}{\mathrm{pa}}
\newcommand{\Desc}{\mathrm{Desc}}
\DeclareMathOperator{\doop}{do}
\newcommand{\Ans}{\mathrm{Ans}}

\title{The Causal Description Gap:\\
Information-Theoretic Separations Across Pearl's Hierarchy}

\author{%
  Seyed Morteza Emadi \\
  Kenan-Flagler Business School \\
  University of North Carolina at Chapel Hill \\
  \texttt{seyed\_emadi@kenan-flagler.unc.edu}%
}

\begin{document}

\maketitle

\begin{abstract}
Pearl's causal hierarchy shows that observational, interventional, and counterfactual queries are qualitatively distinct. We ask a quantitative version of this question: how many additional bits are needed to specify higher-rung causal answers once lower-rung answers are known? We formalize this via \emph{query-class description length}, the Kolmogorov complexity of the answers a structural causal model (SCM) gives to a class of queries, with system size given. Our main construction gives acyclic SCMs on $n$ binary variables whose common observational distribution has constant description length, while single-variable intervention answers can require $\Theta(n^2)$ additional bits. For SCMs with independent noises, fixed finite gate-schema and noise libraries, and at most $d$ parents per variable, a degree-sensitive upper bound gives an observational--interventional gap of at most $O(nd\log(en/d)+n\log n)$. This makes the quadratic construction order-optimal in the dense regime and a rooted-tree construction order-optimal for bounded indegree. The quadratic separation persists under $\varepsilon$-accurate total-variation descriptions for every fixed $\varepsilon<1/4$. At the next rung, even knowing every hard-intervention distribution can leave a $\Theta(n)$ counterfactual description gap; a masked-parity construction extends this to $\Theta(n^2)$. These separation rates hold in the worst case and with high probability under uniform hidden parameters. A general ambiguity-to-bits theorem and Shannon analogue, together with explicit encodings, show that the gaps in our constructions equal the logarithm of residual higher-rung ambiguity up to lower-order terms.
\end{abstract}

%==============================================================================
%==============================================================================
\section{Introduction}
\label{sec:intro}
%==============================================================================

Predictive accuracy identifies what a system does under passive observation. \emph{Causal} questions go further: \emph{interventional} queries (``what if I set $X$?'') and \emph{counterfactual} queries (``what would have happened if $X$ had been different?''). Pearl's causal hierarchy~\citep{pearl2009causality} formalizes these three levels, or rungs (association, intervention, counterfactual), and the Causal Hierarchy Theorem of \citet{bareinboim2022pearl} shows that the hierarchy is generically strict: higher-rung queries cannot in general be answered from lower-rung information. But \emph{how large} is the gap? Is the additional information bounded by a constant, or can it scale with system size?

We give a quantitative answer in bits. We study acyclic structural causal models (SCMs) on $n$ binary variables, so each observation is an $n$-bit vector. An SCM specifies each variable as a function of its direct causes, called parents, and noise. Our baseline models use independent, finite-valued noises with rational probabilities. Hard interventions set selected variables to fixed values. We define \emph{query-class description length} as the Kolmogorov complexity of a model's answers to a class of queries, with $n$ given. A \emph{causal description gap} measures the additional bits needed to specify higher-rung answers when lower-rung answers are already known.

\paragraph{Residual information after non-identification.}
Identification asks whether lower-rung information determines a higher-rung answer. We construct families whose models share the same lower-rung answers but have exponentially many distinct higher-rung answers. Counting these possibilities and matching them with explicit encodings shows that the missing information can grow quadratically with system size.

\paragraph{Contributions.}
\begin{enumerate}
    \item \textbf{Framework and ambiguity principle.} For finite SCM families with finitely encoded query answers, a general counting theorem lower-bounds the worst-case additional bits by the logarithm of the number of compatible higher-rung answers, up to an additive constant. A Shannon analogue and explicit encodings give matching characterizations for our constructions.
    \item \textbf{Observational--interventional separations and optimality.} Rooted-tree and bipartite families have $O(1)$-bit observational descriptions but interventional gaps of $\Theta(n\log n)$ and $\Theta(n^2)$. Let $d$ denote the maximum number of parents of any variable. For fixed finite libraries of noise laws and gate schemas (mechanism rules specified uniformly across input sizes), an upper bound of $O(nd\log(en/d)+n\log n)$ covers all hard-intervention and counterfactual answers. It matches the tree and bipartite gaps for $d=O(1)$ and $d=\Theta(n)$, respectively.
    \item \textbf{Counterfactual separations.} Our modular-XOR and masked-parity families have gaps of $\Theta(n)$ and $\Theta(nd\log(en/d))$, respectively, even given every distribution under hard interventions. The latter gap reaches $\Theta(n^2)$ when $d=\Theta(n)$. Both families assign positive probability to every observational state; no number of hard interventions resolves their counterfactual ambiguity.
    \item \textbf{Robustness, learning, and experiments.} The bipartite family has two layers of $m$ variables and one root, so $n=2m+1$. Its gap survives total-variation error $\varepsilon<1/4$. No observational learner recovers the intervention answers with probability exceeding $2^{-m^2}$ under a uniform hidden graph, regardless of sample size. When each experiment supplies its exact post-intervention distribution, $m$ intervention settings are necessary in the worst case and sufficient for identification.
    \item \textbf{Assumptions and graph knowledge.} With independent noises and positivity (positive probability for every state), a known graph reduces the interventional gap to $O(1)$. Without graph knowledge, adding faithfulness gives a tight worst-case $\Theta(n\log n)$ bound. Faithfulness requires every observational conditional independence to be implied by the graph. Our family with shared latent causes (confounding) retains $\Theta(n)$ bits even given the observed graph.
\end{enumerate}
The separation rates hold in the worst case and with high probability under uniform hidden parameters within the stated families. Table~\ref{tab:summary} summarizes the constructions; Sections~\ref{sec:related} and~\ref{sec:framework} provide related work and the formal framework.

\begin{table}[H]
\centering
\small
\begin{tabular}{@{}lccccl@{}}
\toprule
\textbf{Family} & \textbf{Lower rung} & \textbf{Higher rung} & \textbf{Ambiguity} & \textbf{Gap} & \textbf{Where} \\
\midrule
Rooted trees     & $\Obs$              & $\Int_1$            & $n^{n-1}$ & $\Theta(n \log n)$ & \S\ref{sec:warmup} \\
Bipartite graphs & $\Obs$              & $\Int_1$            & $2^{m^2}$ & $\Theta(n^2)$      & \S\ref{sec:quadratic} \\
Modular XOR      & $\Int_{\mathrm{all}}$ & $\CF_1$           & $2^m$     & $\Theta(n)$        & \S\ref{sec:modular-xor} \\
Masked parity    & $\Int_{\mathrm{all}}$ & $\CF_1$           & $\bigl(\sum_{k=0}^{d}\binom{m}{k}\bigr)^m$ & $\Theta(nd \log\frac{en}{d})$ & \S\ref{sec:masked-parity} \\
Confounded modules & $\Obs, \calG$     & $\Int_1$            & $2^m$     & $\Theta(n)$        & \S\ref{sec:cht} \\
\bottomrule
\end{tabular}
\caption{Lower rung: answers given; higher rung: answers to specify. Ambiguity counts compatible higher-rung answers; gaps are in bits. $\Obs$: observations; $\Int_1$: single-node interventions; $\Int_{\mathrm{all}}$: all hard interventions; $\CF_1$: single-node counterfactuals; $\calG$: observed mixed graph. The bipartite family has $n=2m+1$ variables; the XOR, parity, and confounded families have $n=2m$. For masked parity, $1\le d\le m$.}
\label{tab:summary}
\end{table}

\section{Related Work}
\label{sec:related}
%==============================================================================

\paragraph{Pearl's causal hierarchy.}
Pearl's three-level hierarchy~\citep{pearl2009causality,pearl2018book} is generically strict: the Causal Hierarchy Theorem~\citep{bareinboim2022pearl} establishes non-identifiability, showing that some higher-rung queries are not functions of lower-rung information. \citet{shpitser2008complete} give complete identification algorithms; \citet{ibeling2023comparing} compare frameworks. We quantify this ambiguity: even for simple binary acyclic SCMs, specifying a higher-rung answer given the lower-rung answer can require $\Theta(n^2)$ additional bits.

\paragraph{Causal discovery.}
Constraint- and score-based methods identify causal graphs up to Markov equivalence from observations~\citep{spirtes2000causation,peters2017elements,hauser2012characterization}. \citet{eberhardt2005number} and \citet{shanmugam2015learning} bound the number/sample complexity of interventions for identification. We give description-length lower bounds: even with infinite observational data, the obs--int information gap is $\Theta(n^2)$ in our family.

\paragraph{Description length and MDL.}
Kolmogorov complexity~\citep{kolmogorov1965three,li2008introduction} and minimum description length (MDL)~\citep{grunwald2007minimum} measure the information needed to describe objects or models. Our use is different: the compressed object is neither the dataset nor the SCM parameterization, but the complete collection of answers to a query class, or \emph{answer oracle}. This distinction is what enables our separations: the observational answer oracle can have constant description length even when the corresponding interventional oracle requires $\Theta(n^2)$ bits.

\paragraph{Causal reasoning in AI.}
Recent benchmarks find uneven causal-reasoning performance in LLMs~\citep{kiciman2023causal,jin2024cladder}; our results give one information-theoretic reason why observational predictive training alone need not induce interventional competence in worst-case structured families.

\paragraph{Complexity across Pearl's hierarchy.}
\citet{dorfler2024complexity} study satisfiability across probabilistic, interventional, and counterfactual languages in Pearl's hierarchy, proving strictly increasing computational complexity for certain languages (e.g., NP$^{\mathrm{PP}}$, PSPACE, and NEXP completeness under addition and marginalization). They ask how hard it is to decide whether there exists an SCM satisfying a given collection of causal statements; we bound the \emph{information content} of higher-rung answer oracles given lower-rung ones.

%==============================================================================
\section{Framework: Query-Class Description Length}
\label{sec:framework}
%==============================================================================

We first define the models and their causal answers, then measure the additional bits needed to specify higher-rung answers given lower-rung answers. The lemma proofs are in Appendix~\ref{app:framework}.

\subsection{Models and Causal Answers}
\label{sec:scm-review}
\label{sec:query-families}

\begin{definition}[Structural Causal Model]
\label{def:scm}
A \emph{structural causal model} (SCM) on binary variables $X_1, \ldots, X_n \in \{0,1\}$ is a tuple $M = (G, \{f_i\}, P_U)$ where:
\begin{itemize}
    \item $G$ is a directed acyclic graph (DAG) on $[n] := \{1, \ldots, n\}$.
    \item $U_1, \ldots, U_n$ are independent finite-valued exogenous (noise) variables; their product law $P_U$ is specified by finite rational probability tables.
    \item Structural equations specify each variable as a function of its parents $\pa_G(i)$ and noise:
    \[
    X_i = f_i\bigl( (X_j)_{j \in \pa_G(i)}, U_i \bigr).
    \]
\end{itemize}
We require each $f_i$ to be computable. Every distribution defined below is then a finite table of rational probabilities computable from the model.
\end{definition}

Sampling $(U_1, \ldots, U_n) \sim P_U$ and evaluating the equations in topological order gives the \textbf{observational distribution} $P_M(X_1, \ldots, X_n)$.

\begin{definition}[Hard intervention]
A \emph{hard intervention} $\doop(X_i = x)$ modifies $M$ by replacing the equation for $X_i$ with the constant $X_i := x$, leaving all other equations unchanged. The resulting distribution is $P_M^{\doop(X_i = x)}$.
\end{definition}

\begin{definition}[Counterfactual]
For a fixed exogenous realization $u = (u_1, \ldots, u_n)$, the \emph{counterfactual} $X^{(i \leftarrow b)}(u)$ is the value of the entire system $(X_1, \ldots, X_n)$ when we intervene with $\doop(X_i = b)$ but use the \emph{same} noise realization $u$.

Counterfactual queries concern joint distributions such as $P(X, X^{(i \leftarrow 0)}, X^{(i \leftarrow 1)})$ of three $n$-bit vectors: the actual outcome and those with $X_i$ forced to $0$ and to $1$, computed from the same noise.
\end{definition}

We collect these distributions into answer objects, also called answer oracles.

\begin{definition}[Query Families]
\label{def:query-families}
For an SCM $M$ on $n$ variables:
\begin{itemize}
    \item \textbf{Observational:} $\Obs(M) := P_M$, the joint distribution over $(X_1, \ldots, X_n)$.
    
    \item \textbf{Single-node interventional:} $\Int_1(M) := \bigl( P_M,\, (P_M^{\doop(X_i = 0)}, P_M^{\doop(X_i = 1)})_{i=1}^n \bigr)$, the observational distribution plus the $2n$ single-variable interventional distributions.

    \item \textbf{All hard interventions:} $\Int_{\mathrm{all}}(M) := \bigl(P_M^{\doop(X_S = x_S)}\bigr)_{S \subseteq [n],\, x_S \in \{0,1\}^S}$, the laws under all hard interventions on subsets of variables, single-node and multi-node. We call it the \emph{full hard-do interventional oracle} (it excludes soft, stochastic, edge, and exogenous interventions); $\Int_1(M)$ is computable from it.

    \item \textbf{Single-node counterfactual:} $\CF_1(M) := \bigl( P(X, X^{(i \leftarrow 0)}, X^{(i \leftarrow 1)}) \bigr)_{i=1}^n$. This collects the counterfactual joint distributions for all variables.
\end{itemize}
\end{definition}

\subsection{Description Lengths and Causal Gaps}
\label{sec:description-lengths}

\begin{definition}[Kolmogorov Complexity]
Fix a universal prefix-free Turing machine. The \emph{Kolmogorov complexity} $K(z)$ is the length of the shortest program that outputs string $z$. The conditional complexity $K(z \mid w)$ is the shortest program length when $w$ is given as auxiliary input. By the invariance theorem~\citep[Section~2.1]{li2008introduction} and the same argument for encodings, neither the choice of universal machine nor the choice of computably invertible encoding changes $K$ by more than an additive $O(1)$; we encode a distribution as the string of its lowest-terms rational probabilities.
\end{definition}

\begin{definition}[Query-Class Description Length: General Form]
\label{def:qcdl-general}
Let $Q$ be a query class with a computable finite encoding $\Ans_Q(M)$ of all its answers under $M$. The \emph{query-class description length} of $M$ with respect to $Q$ is
\[
\DL(Q;\, M) \;:=\; K\!\bigl(\Ans_Q(M) \mid n\bigr).
\]
The \emph{conditional gap} between query classes $Q_2$ and $Q_1$ is
\[
\Delta(Q_2 \mid Q_1;\, M) := K\!\bigl(\Ans_{Q_2}(M) \mid \Ans_{Q_1}(M),\, n\bigr).
\]
The lower-rung answers and the number of variables are supplied as inputs; the gap counts the additional program bits needed to produce the higher-rung answers.
\end{definition}

For the observational and single-node query families, write
$\DL_1(M) := K(\Obs(M) \mid n)$,
$\DL_2(M) := K(\Int_1(M) \mid n)$, and
$\DL_3(M) := K(\CF_1(M) \mid n)$.
The corresponding \emph{causal description gaps} are $\Delta_{2|1}(M) := K(\Int_1(M) \mid \Obs(M), n)$ and $\Delta_{3|2}(M) := K(\CF_1(M) \mid \Int_1(M), n)$. We also consider $K(\CF_1(M) \mid \Int_{\mathrm{all}}(M), n)$, the counterfactual gap given all hard-intervention distributions.

The following lemmas supply the counting and reduction tools for our separations.

\begin{lemma}[Counting bound for Kolmogorov complexity]
\label{lem:counting}
Let $z_1, \ldots, z_N$ be $N$ distinct strings and $w$ an arbitrary auxiliary string. Then:
\begin{enumerate}
    \item At least one $z_i$ satisfies $K(z_i \mid w) \geq \log_2 N - O(1)$.
    \item For every $c \geq 0$, all but at most a $2^{-c}$ fraction of the $z_i$ satisfy $K(z_i \mid w) \geq \log_2 N - c - O(1)$.
\end{enumerate}
\end{lemma}

We use this bound when distinct hidden structures give distinct answers.

\begin{samepage}
\begin{lemma}[Hierarchy Monotonicity]
\label{lem:monotonicity}
$\DL_1(M) \leq \DL_2(M) + O(1) \leq \DL_3(M) + O(1)$.
\end{lemma}

Selecting components and taking marginals recovers lower-rung answers from higher-rung answers.
\end{samepage}

\begin{lemma}[Conditioning on a constant-complexity object is free]
\label{lem:conditioning-simple}
If a family satisfies $K(\Obs(M) \mid n) = O(1)$ with a bound independent of $M$ and $n$, then $K(\Int_1(M) \mid \Obs(M), n) = K(\Int_1(M) \mid n) \pm O(1)$, and consequently $\Delta_{2|1}(M) = \DL_2(M) \pm O(1)$.
\end{lemma}

In the tree and bipartite families, observations have constant complexity, so by Lemma~\ref{lem:conditioning-simple} interventional description length gives the gap. For arbitrary lower-rung answers, the conditional definition applies directly.

%==============================================================================
\section{Observational--Interventional Separations}
\label{sec:interventional}
%==============================================================================

We apply the counting bound to families whose members share one observational law: trees give the bounded-indegree separation, and bipartite graphs give the quadratic separation. Full proofs are in Appendix~\ref{app:interventional}, and Figure~\ref{fig:constructions}(a) and (b) there sketch the two constructions.

\subsection{Rooted Trees}
\label{sec:warmup}

Let $T$ be a rooted labeled tree on $[n]$, with root $r$ and parent $\pa_T(v)$ for each $v\neq r$. Set $X_r := U_r \sim \mathrm{Bernoulli}(1/2)$ and $X_v := X_{\pa_T(v)}$ for $v\neq r$. All trees share the observational law $P^\star$, assigning mass $1/2$ to each of $0^n$ and $1^n$.

Under $\doop(X_i=0)$, every descendant of $i$, including $i$ itself, becomes $0$; non-descendants remain fair. The distributions for these $n$ intervention settings therefore identify the descendant sets, which determine the root and parent links.

\begin{theorem}[Rooted-tree separation]
\label{thm:tree-main}
Let $\calM_{\mathrm{tree}}^n$ be the family of rooted-tree SCMs $M_T$ above. All members share $\Obs(M_T) = P^\star$ with $\DL_1(M_T) = O(1)$. Moreover:
\begin{enumerate}
    \item \textbf{Upper bound.} For every $T$, $\DL_2(M_T) \leq (n-1)\log_2 n + O(\log n)$.
    \item \textbf{Lower bound.} For uniformly random $T$ on $[n]$ and every $c \geq 0$,
    \[
        \Pr\!\bigl[\Delta_{2|1}(M_T) \geq (n-1)\log_2 n - c - O(1)\bigr] \;\geq\; 1 - 2^{-c}.
    \]
\end{enumerate}
\end{theorem}

Thus a uniformly chosen tree has a $\Theta(n\log n)$ gap with high probability, although every node has at most one parent.

\emph{Proof sketch.} Parent labels describe $T$ in $(n-1)\log_2 n + O(\log n)$ bits. Cayley's formula gives $n^{n-1}$ rooted trees; $\doop(X_i = 0)$ zeroes exactly $i$ and its descendants, so the intervention tables determine $T$. Lemma~\ref{lem:counting} gives the lower bound; Lemma~\ref{lem:conditioning-simple} moves it to $\Delta_{2|1}$ since $P^\star$ has constant complexity.

\subsection{Bipartite Graphs}
\label{sec:quadratic}

To obtain a quadratic gap, we use two layers of $m$ variables and a common root, so $n=2m+1$. Each of the $m^2$ possible edges between the layers can be present or absent.

\begin{definition}[Bipartite-graph SCM]
\label{def:bipartite-scm}
Let $m\geq1$ and $n=2m+1$. Use a fixed computable partition into a root $r$, layer $A = \{a_1, \dots, a_m\}$, and layer $B = \{b_1, \dots, b_m\}$. For a bipartite graph $G \subseteq A \times B$, the SCM $M_G$ has $U_r \sim \mathrm{Bernoulli}(1/2)$ as the only random exogenous variable, and structural equations
\[
X_r := U_r, \qquad X_{a_i} := X_r \;\;\text{for all } a_i \in A, \qquad X_{b_j} := X_r \wedge \!\!\bigwedge_{(a_i, b_j) \in G}\!\! X_{a_i} \;\;\text{for all } b_j \in B,
\]
where the empty AND is defined as $1$.
\end{definition}

Every variable equals the root observationally. Setting a layer-$A$ variable to zero forces its neighbors in $B$ to zero, while non-neighbors remain fair. The following lemma makes this distinction precise.

\begin{lemma}[Bipartite decoding]
\label{lem:bipartite-obs}\label{lem:bipartite-intervention}\label{lem:bipartite-injective}
For every bipartite graph $G \subseteq A \times B$:

(i) $\Obs(M_G) = P^\star$, assigning mass $1/2$ to each of $0^n$ and $1^n$.

(ii) For every $a_i \in A$, $P_{M_G}^{\doop(X_{a_i} = 0)}(X_{b_j} = 0)$ equals $1$ if $(a_i, b_j) \in G$ and $1/2$ otherwise, so the neighborhood $N_G(a_i) := \{b_j : (a_i, b_j) \in G\}$ is determined by $P_{M_G}^{\doop(X_{a_i} = 0)}$.

(iii) Consequently the map $G \mapsto \Int_1(M_G)$ is injective.
\end{lemma}

The distributions for the $m$ settings $\doop(X_{a_i}=0)$ therefore determine all $m^2$ adjacency bits.

\begin{theorem}[Quadratic separation: $\Theta(n^2)$]
\label{thm:quadratic-separation}
Let $n = 2m+1$ and $\calM_{\mathrm{bip}}^n := \{M_G : G \subseteq A \times B\}$ as in Definition~\ref{def:bipartite-scm}. Then:
\begin{enumerate}
    \item \textbf{Upper bound.} For every $G$, $\DL_2(M_G) \leq m^2 + O(1)$.
    \item \textbf{Lower bound.} For $G \sim \mathrm{Unif}(2^{A \times B})$ and every $c \geq 0$,
    \[
        \Pr\bigl[ \Delta_{2|1}(M_G) \geq m^2 - c - O(1) \bigr] \geq 1 - 2^{-c}.
    \]
    In particular, taking $c = m^2/2$ shows that a uniformly random graph has $\Delta_{2|1}(M_G) = \Omega(n^2)$ except with probability $\exp(-\Omega(n^2))$.
\end{enumerate}
\end{theorem}

\emph{Proof sketch.} Encoding the adjacency matrix uses $m^2$ bits and lets a fixed program simulate $M_G$. Conversely, Lemma~\ref{lem:bipartite-injective}(iii) gives $2^{m^2}$ distinct interventional answers with the same observation $P^\star$. Applying Lemma~\ref{lem:counting} conditional on $(P^\star,n)$ gives the lower bound.

\subsection{Degree-Sensitive Optimality}
\label{sec:degree-optimality}

The tree and bipartite constructions differ in their maximum indegree. To bound the gap in terms of indegree, we also fix the available mechanisms and noise laws; describing the model then reduces to encoding parent sets and choices from these libraries.

\begin{definition}[Finite-gate-schema bounded-indegree SCM class]
\label{def:finite-gate-class}
A \emph{gate schema} is a computable family $\gamma = (\gamma_k)_{k \geq 0}$ with $\gamma_k : \{0,1\}^k \times \mathcal{U}_\gamma \to \{0,1\}$, specified by a constant-length program independent of $n$ and $k$ (uniformly across input sizes). Here $\mathcal{U}_\gamma$ is the noise-input space. Examples include the constant, copy, negation, parity, and unbounded-AND/OR schemas.

Fix a \emph{gate library}, a finite set $\Gamma$ of gate schemas, and a \emph{noise library}, a finite set $\Pi$ of finite-support rational noise distributions, both independent of $n$. The class $\calM_{n,d}(\Gamma,\Pi)$ consists of binary acyclic SCMs on $n$ endogenous variables, each with at most $d$ endogenous parents. Each equation applies a component $\gamma_k$, $\gamma\in\Gamma$, to its $k\leq d$ parent values in increasing variable-label order and its local noise; each local noise law belongs to $\Pi$.
\end{definition}

The bipartite-graph SCMs of Definition~\ref{def:bipartite-scm} belong to $\calM_{n,m+1}(\Gamma_{\mathrm{bip}}, \Pi_{\mathrm{bip}})$ for $\Gamma_{\mathrm{bip}} = \{\mathrm{copy}, \mathrm{AND}\}$ and $\Pi_{\mathrm{bip}} = \{\delta_0, \mathrm{Bernoulli}(1/2)\}$, where $\delta_0$ is the point mass at $0$; the AND schema applies uniformly at every input size. The rooted-tree SCMs belong to $\calM_{n,1}(\{\mathrm{copy}\}, \{\delta_0, \mathrm{Bernoulli}(1/2)\})$.

\begin{theorem}[Degree-sensitive upper bound]
\label{thm:degree-upper}
For every fixed finite gate-schema library $\Gamma$ and finite noise library $\Pi$, and every $1 \leq d \leq n-1$, there exists a constant $C_{\Gamma,\Pi}$ such that every $M \in \calM_{n,d}(\Gamma,\Pi)$ satisfies
\[
    \DL_2(M) \;\leq\; C_{\Gamma,\Pi} \cdot n + n \log_2\!\Bigl(\sum_{k=0}^{d} \binom{n-1}{k}\Bigr) + O(n \log n).
\]
In particular $\DL_2(M) = O(nd \log\frac{en}{d} + n \log n)$, and consequently $\Delta_{2|1}(M)$ obeys the same bound.
\end{theorem}

\emph{Proof sketch.} An SCM description encodes a topological order, parent sets, and gate/noise choices; a fixed program then computes its intervention answers. The inequality $\sum_{k=0}^{d}\binom{n-1}{k} \leq (d+1)(e(n-1)/d)^d$ for $1 \leq d \leq n-1$ gives the asymptotic bound.

\begin{corollary}[Optimality of the tree and bipartite constructions]
\label{cor:degree-optimality}
Every class $\calM_{n,d}(\Gamma,\Pi)$ has gap at most $O(n\log n)$ when $d=O(1)$ and $O(n^2)$ when $d=\Theta(n)$. The rooted-tree and bipartite constructions, which lie in such classes, attain these rates by Theorems~\ref{thm:tree-main} and~\ref{thm:quadratic-separation}, so they are order-optimal within finite-gate-schema SCMs of maximum indegree $d$ in the bounded-degree and dense regimes.
\end{corollary}

\begin{corollary}[The same rate for every oracle computable from the SCM]
\label{cor:int-all-rate}
For every $M \in \calM_{n,d}(\Gamma,\Pi)$ and every fixed query class $Q$ with a finite answer computable from a description of $M$, including $\Int_{\mathrm{all}}$ and $\CF_1$, $\DL(Q; M) = O(nd\log\frac{en}{d} + n\log n)$. For every SCM, $K(\Int_{\mathrm{all}} \mid \Obs, n) \geq K(\Int_1 \mid \Obs, n) - O(1)$, so all lower bounds for $\Int_1$ transfer: $\Delta(\Int_{\mathrm{all}} \mid \Obs)$ is $\Theta(n^2)$ for all but an $\exp(-\Omega(n^2))$ fraction of bipartite graphs and $\Theta(n \log n)$ for all but an exponentially small fraction of trees.
\end{corollary}

Without the library restriction, the gap can be unbounded even for fixed $n$ and $d$ (Remark~\ref{rem:bounded-library}).

\subsection{Finite-Precision Robustness}
\label{sec:finite-precision}

We now allow total-variation error in the specified intervention distributions.

\begin{definition}[Approximate query-class description length]
\label{def:approx-dl}
Let $d_Q$ be a metric on answer objects for query class $Q$. For $\varepsilon \geq 0$, define
\[
    K_\varepsilon\bigl(\Ans_Q(M) \mid w\bigr) := \min\bigl\{ |p| \;:\; d_Q\!\bigl(U(p, w), \Ans_Q(M)\bigr) \leq \varepsilon \bigr\},
\]
where $U$ is the fixed universal machine and $w$ is auxiliary input. The minimum is over halting programs outputting valid encoded answer objects. The $\varepsilon$-approximate description gap is $\Delta_\varepsilon(Q_2 \mid Q_1; M) := K_\varepsilon(\Ans_{Q_2}(M) \mid \Ans_{Q_1}(M), n)$. For interventional answer objects we use the metric
\[
    d_{\Int}(I, I') := \max\!\left\{ \mathrm{TV}(P_I, P_{I'}),\; \max_{i \in [n],\, b \in \{0,1\}} \mathrm{TV}\!\bigl( P^{\doop(X_i = b)}_I,\, P^{\doop(X_i = b)}_{I'} \bigr) \right\},
\]
where $\mathrm{TV}$ denotes total variation distance and $P_I, P_{I'}$ are the observational components of $I, I'$.
\end{definition}

\begin{samepage}
\begin{theorem}[Approximate quadratic separation]
Consider the bipartite family $\calM_{\mathrm{bip}}^n$ with $n = 2m + 1$. For every fixed $0\leq\varepsilon < 1/4$:\label{thm:approx-quadratic}
\begin{enumerate}
    \item \textbf{Upper bound.} For every $G$, $\Delta_\varepsilon(\Int_1 \mid \Obs;\, M_G) \leq m^2 + O(1)$.
    \item \textbf{Lower bound (existence).} There exists $G \subseteq A \times B$ such that
    \[
        \Delta_\varepsilon(\Int_1 \mid \Obs;\, M_G) \;\geq\; m^2 - O(1).
    \]
    \item \textbf{Lower bound (high probability).} For uniformly random $G \sim \mathrm{Unif}(2^{A \times B})$ and every $c \geq 0$,
    \[
        \Pr\!\left[\, \Delta_\varepsilon(\Int_1 \mid \Obs;\, M_G) \geq m^2 - c - O(1) \,\right] \;\geq\; 1 - 2^{-c}.
    \]
\end{enumerate}
\end{theorem}

\end{samepage}

The quadratic information requirement therefore survives constant error in every intervention distribution.

\emph{Proof sketch.} Exact encoding gives the upper bound. If $G\neq G'$, some edge differs, so the corresponding intervention gives marginal probabilities differing by $1/2$. Their answer objects are therefore at least $1/2$ apart in $d_{\Int}$, making the $\varepsilon$-balls disjoint for $\varepsilon<1/4$. Counting conditional on $(P^\star,n)$ gives the lower bounds.

%==============================================================================
\section{Counterfactual Separations: Modular XOR and Masked Parity}
\label{sec:counterfactual}
%==============================================================================

Section~\ref{sec:interventional} quantified what observations leave unspecified about interventions. We now construct models that agree on every hard-intervention distribution but differ in counterfactual answers. Modular XOR gives a linear gap; varying parent sets gives larger gaps. Full proofs for this section are in Appendix~\ref{app:xor-proofs}.

\subsection{The Modular-XOR Family}
\label{sec:modular-xor}

Consider a pair with a fair-bit input and an output that either equals independent fair noise or XORs that noise with the input. Both mechanisms give identical hard-intervention distributions. With the noise held fixed, changing the input leaves the output unchanged in the first mechanism and flips it in the second.

\begin{definition}[Modular-XOR SCM]
\label{def:xor-scm}
Let $m\geq1$ and $n=2m$. For a hidden string $s \in \{0,1\}^m$, the SCM $M_s$ has $m$ independent modules $(X_t, Y_t)$ with i.i.d.\ exogenous $U_{X_t}, U_{Y_t} \sim \mathrm{Bernoulli}(1/2)$ and
\[
X_t := U_{X_t}, \qquad Y_t := \begin{cases} U_{Y_t} & \text{if } s_t = 0, \\ X_t \oplus U_{Y_t} & \text{if } s_t = 1. \end{cases}
\]
Here $\oplus$ denotes XOR, addition modulo two.
\end{definition}

\newpage
\begin{lemma}[Modular-XOR invariances]
\label{lem:xor-obs}\label{lem:xor-int}\label{lem:xor-all-int}\label{lem:xor-cf}
For all $s, s' \in \{0,1\}^m$:

(i) $\Obs(M_s)$ is uniform on $\{0,1\}^{2m}$.

(ii) $\Int_1(M_s) = \Int_1(M_{s'})$.

(iii) $\Int_{\mathrm{all}}(M_s) = \Int_{\mathrm{all}}(M_{s'})$, so the counterfactual gap below survives complete knowledge of all hard do-distributions.

(iv) The map $s \mapsto \CF_1(M_s)$ is injective.
\end{lemma}

\begin{theorem}[Counterfactual gap after the full hard-do interventional oracle]
\label{thm:cf-separation}
Let $n=2m$ and $\calM_{\mathrm{xor}}^n:=\{M_s:s\in\{0,1\}^m\}$ as in Definition~\ref{def:xor-scm}. For every $s$, $\DL(\Int_{\mathrm{all}};M_s)=O(1)$ and $\DL_3(M_s)\leq m+O(1)$. For uniform $s$ and every $c\geq0$,
\[
\Pr\!\left[K\!\left(\CF_1(M_s)\mid\Int_{\mathrm{all}}(M_s),n\right)\geq m-c-O(1)\right]\geq1-2^{-c}.
\]
\end{theorem}

For uniform $s$, the gap is $\Theta(n)$ with high probability, even given $\Int_{\mathrm{all}}$; this also holds for $\Delta_{3|2}$.

\emph{Proof sketch.} Given $n$, a fixed program generates the common oracle: intervention targets are fixed and all remaining variables are independent fair bits. Lemma~\ref{lem:xor-cf}(iv) gives $2^m$ distinct counterfactual answers. Conditional counting gives the lower bound; encoding $s$ in $m$ bits gives the upper bound.

\subsection{Masked Parity: A Degree-Sensitive Counterfactual Gap}
\label{sec:masked-parity}

Each modular-XOR output has two parent-set choices: empty or its designated input. Allowing any input gives $m+1$ choices at $d=1$; allowing sets of up to $d$ inputs increases the ambiguity further. Each output is the parity of its chosen inputs, masked by an independent fair bit.

Let $n=2m$, $1\leq d\leq m$, and $S=(S_1,\ldots,S_m)$ with $S_j\subseteq[m]$ and $|S_j|\leq d$. Define $M_S$ by $X_i:=U_i$ and $Y_j:=\bigl(\bigoplus_{i\in S_j}X_i\bigr)\oplus V_j$, with independent fair bits $U_i,V_j$ and the empty XOR equal to $0$.

The family $\calM_{\mathrm{par}}^{n,d}$ has $N_{m,d}:=\bigl(\sum_{k=0}^{d}\binom{m}{k}\bigr)^m$ members. It uses only the parity schema and fair-bit noise, and contains modular XOR when $S_j\subseteq\{j\}$.

\begin{theorem}[Degree-sensitive counterfactual gap]
\label{thm:parity-separation}
Let $L_{m,d}:=\log_2 N_{m,d}=\Theta(nd\log\frac{en}{d})$. Every $M_S$ has uniform $\Obs(M_S)$ on $\{0,1\}^{2m}$ and the same $\Int_{\mathrm{all}}(M_S)$, with $\DL(\Int_{\mathrm{all}};M_S)=O(1)$.

For every $S$, $\DL_3(M_S)\leq L_{m,d}+O(\log n)$. For uniform $S$ and every $c\geq0$,
\[
\Pr\!\left[K(\CF_1(M_S)\mid\Int_{\mathrm{all}}(M_S),n)\geq L_{m,d}-c-O(1)\right]\geq1-2^{-c}.
\]
With high probability under uniform $S$, the gap is $\Theta(n\log n)$ at $d=1$ and $\Theta(n^2)$ at $d=\Theta(n)$, order-optimal within finite-gate-schema SCMs of maximum indegree $d$ by Corollary~\ref{cor:int-all-rate}.
\end{theorem}

Quadratic counterfactual gaps can therefore persist even given every hard-intervention law, and with uniform, full-support observations.

\emph{Proof sketch.} Private fair noise makes all hard-intervention laws identical across $S$. Holding the noise fixed gives $Y_j^{(X_i\leftarrow0)}\oplus Y_j^{(X_i\leftarrow1)}=\mathbf{1}[i\in S_j]$, so $\CF_1$ determines $S$. Counting the $N_{m,d}$ distinct answers gives the lower bound; encoding the index of $S$ and $d$ gives the upper bound.

%==============================================================================
\section{Ambiguity-to-Bits: A General Principle}
\label{sec:general}
%==============================================================================

The constructions in Sections~\ref{sec:interventional}--\ref{sec:counterfactual} follow a common principle: count the distinct higher-rung answers compatible with a lower-rung answer. The logarithm of this count, up to an additive constant, lower-bounds the worst-case additional description length. We formalize this principle and its Shannon counterpart, then derive consequences for observational learning and intervention design. Full proofs for this section are in Appendix~\ref{app:general}.

\begin{definition}[General ambiguity class]
\label{def:gen-ambiguity}
For query classes $Q_1, Q_2$, a finite SCM family $\calM$ on a common set of $n$ variables, and an attained lower-rung answer $a$, define
\[
    \calF_{Q_2 \mid Q_1}(a; \calM) := \{\Ans_{Q_2}(M) : M \in \calM,\, \Ans_{Q_1}(M) = a\}.
\]
\end{definition}

\begin{theorem}[Ambiguity-to-bits, Kolmogorov form]
\label{thm:ambiguity-bits}
For any finite family $\calM$ on $n$ variables, query classes $Q_1,Q_2$, and attained lower-rung answer $a$,
\[
    \max_{M \in \calM:\, \Ans_{Q_1}(M) = a} K\!\bigl(\Ans_{Q_2}(M) \mid a, n\bigr) \;\geq\; \log_2 |\calF_{Q_2 \mid Q_1}(a; \calM)| - O(1).
\]
Moreover, for every $c \geq 0$, all but a $2^{-c}$ fraction of the distinct strings in $\calF_{Q_2 \mid Q_1}(a; \calM)$ have conditional complexity at least $\log_2 |\calF_{Q_2 \mid Q_1}(a; \calM)| - c - O(1)$.
\end{theorem}

The probability statement is uniform over distinct higher-rung answers; injectivity transfers it to uniform hidden parameters in our constructions.

\begin{theorem}[Ambiguity-to-entropy, Shannon form]
\label{thm:ambiguity-entropy}
Let $\Theta$ be a finite-valued random hidden parameter and $M_\Theta$ its SCM. Write $L(\Theta):=\Ans_{Q_1}(M_\Theta)$ and $U(\Theta):=\Ans_{Q_2}(M_\Theta)$ for the lower- and higher-rung answers, and let $\mathsf{H}$ denote discrete Shannon entropy. If $U$ distinguishes all parameter values consistent with each lower-rung answer, then $\mathsf{H}(U(\Theta)\mid L(\Theta)=\ell)=\mathsf{H}(\Theta\mid L(\Theta)=\ell)$ for every $\ell$ with $\Pr[L(\Theta)=\ell]>0$. In particular, if $\Theta$, conditional on $L(\Theta)=\ell$, is uniform over $N$ compatible parameter values, this conditional entropy equals $\log_2 N$.
\end{theorem}

Under this condition, identifying the higher-rung answer resolves all remaining uncertainty about the hidden parameter.

In Sections~\ref{sec:interventional}--\ref{sec:counterfactual}, the explicit encodings match log ambiguity up to lower-order terms, as summarized in Table~\ref{tab:summary}. Under uniform hidden parameters, the same separation rates therefore hold for Shannon conditional entropy.

\begin{remark}[Computability]
\label{rem:computability}
Kolmogorov complexity is uncomputable, and exact ambiguity counting can be hard for implicit families. Our constructions use explicit counts and encodings; efficient enumeration in general remains open.
\end{remark}

\subsection{Learning-Theoretic Consequence}
\label{sec:learning-main}

In the bipartite family, passive data reveal nothing about the hidden graph.

\begin{corollary}[No-free-lunch for observational learners]
\label{cor:nfl-main}
Let $G$ be uniform in the bipartite family and $D$ an i.i.d.\ observational sample of any size. Their mutual information is $I(G;D)=0$. Every learner using only $D$ and private randomness independent of $G$ has probability at most $2^{-m^2}$ of exactly recovering $\Int_1(M_G)$. Given also an independent uniform queried pair $(i,j)\in[m]^2$, every such learner predicting $P^{\doop(X_{a_i}=0)}_{M_G}(X_{b_j}=0)$ has expected absolute error at least $1/4$.
\end{corollary}

More observational data cannot reduce this family's causal ambiguity. Benchmarks supplying causal graphs or narratives, such as CLadder~\citep{jin2024cladder}, give the learner additional information, so Corollary~\ref{cor:nfl-main} does not directly apply to that setting.

\subsection{From Bits to Experiments}
\label{sec:experiments}

To translate ambiguity into experiment counts, fix a finite SCM family and a lower-rung answer with $N$ compatible target oracles. Each \emph{experiment} is a hard-intervention setting that returns its exact post-interventional law; let $R$ be the maximum number of distinct returned distributions across compatible models for any one setting, so one setting reveals at most $\log_2 R$ bits. We count settings; estimating each law may require many samples.

\begin{proposition}[From bits to experiments]
\label{prop:bits-to-experiments}
Any adaptive strategy guaranteed to identify the target on every compatible model in at most $t$ settings satisfies $R^t\geq N$. Hence $t\geq\lceil\log_2 N/\log_2 R\rceil$ when $R>1$, and identification is impossible when $R=1<N$.

In the bipartite family $N=2^{m^2}$ and $R\leq2^m$, so $m$ settings are necessary, and the $m$ single-node interventions $\doop(X_{a_i}=0)$ suffice. Adaptivity and multi-node interventions cannot reduce this worst-case setting count. In the modular-XOR and masked-parity families $R=1$, so no number of hard-intervention experiments identifies the counterfactual oracle.
\end{proposition}

\emph{Proof sketch.} A depth-$t$ strategy has at most $R^t$ outcome sequences. A hard intervention on a bipartite SCM sees $G$ only through one bit per non-intervened layer-$B$ node. In the XOR and parity families, all hard-intervention laws coincide, so these experiments cannot resolve the counterfactual ambiguity.

%==============================================================================
\section{How Assumptions Change the Interventional Gap}
\label{sec:cht}
%==============================================================================

Experiments can reduce ambiguity by supplying new distributions; assumptions can reduce it by restricting the compatible models. We now examine graph knowledge, positivity, faithfulness, and shared latent causes. Full statements and proofs are in Appendix~\ref{app:assumptions}.

\begin{proposition}[What assumptions buy]
\label{prop:assumptions}
Parts 1--2 concern binary SCMs with independent local noises (Definition~\ref{def:scm}) and do not require fixed finite gate libraries. Part 3 permits shared latent causes.
\begin{enumerate}
    \item \textbf{Graph knowledge.} In the tree and bipartite families, $\Delta_{2|1}(M_G)=K(G\mid n)\pm O(1)$. More generally, a known graph $G$ and positive probability for every state give $K(\Int_{\mathrm{all}}\mid\Obs,G,n)=O(1)$. Without positivity, this conclusion can fail.
    \item \begin{samepage}\textbf{Positivity and faithfulness.} If $P=\Obs(M)$ is strictly positive and faithful to the graph, let $\calE(P)$ be the set of DAGs implying exactly the conditional independences of $P$. Then $K(\Int_{\mathrm{all}}\mid\Obs,n)\leq\log_2|\calE(P)|+O(\log n)\leq\log_2(n!)+O(\log n)$. For generic positive rational $P$ (outside a proper algebraic exceptional set), the $n!$ orderings yield faithful SCMs on complete DAGs sharing $P$ but with distinct $\Int_1$ oracles. Under uniform orderings, their conditional Shannon entropy is $\log_2(n!)$ and their conditional complexity is $\Theta(n\log n)$ with high probability.
    \par\end{samepage}
    \item \textbf{Confounding.} We construct $2^m$ SCMs on $n=2m$ variables, indexed by $m$-bit strings, using independent finite latent roots shared by pairs of variables. All have the same uniform observations and observed mixed graph $\calG$, which records directed edges and pairs sharing a latent cause. For every hidden string, $K(\Int_1\mid\Obs,\calG,n)\leq m+O(1)$. Under a uniform hidden string, $\mathsf{H}(\Int_1\mid\Obs,\calG)=m$ and, for every $c\geq0$, the Kolmogorov gap is at least $m-c-O(1)$ with probability at least $1-2^{-c}$. The lower bound also holds for $\varepsilon$-accurate answers in the total-variation metric of Definition~\ref{def:approx-dl} when $0\leq\varepsilon<1/8$.
\end{enumerate}
\end{proposition}

A known graph identifies interventions under positivity and independent noise, while unknown orientations or shared latent causes can leave growing gaps.

\paragraph{Linear models.}
The $\Theta(n\log n)$ worst-case gap also occurs in faithful zero-mean linear Gaussian models with rational covariance matrices. Computable identification from a finite observational encoding reduces the gap to $O(1)$. Linear Gaussian models with independent errors of equal positive variance satisfy these conditions under rational covariance encoding~\citep{peters2014identifiability}. Linear non-Gaussian acyclic models (LiNGAM) with independent errors are identifiable~\citep{shimizu2006linear}; the same description-gap conclusion additionally requires a finite encoding and computable recovery.

%==============================================================================
\section{Conclusion}
\label{sec:conclusion}
%==============================================================================

We quantify the additional information needed to move between levels of Pearl's causal hierarchy. Our binary SCM families exhibit $\Theta(n^2)$ gaps from observations to interventions and from all hard-intervention laws to counterfactuals, in the worst case and with high probability under uniform hidden parameters. The former starts from an $O(1)$-bit observational description, survives total-variation error $\varepsilon<1/4$, and is order-optimal within finite-gate-schema SCMs of maximum indegree $d=\Theta(n)$. General ambiguity bounds and matching encodings explain these gaps in Kolmogorov and Shannon terms. Experiment counts also depend on how many distinct distributions each setting can return. A known graph with positivity closes the interventional gap; positivity and faithfulness alone leave $\Theta(n \log n)$ bits; shared latent causes keep $\Theta(n)$ bits even given the observed graph. Further work could extend these bounds to nonlinear additive-noise and smooth models and to soft and stochastic interventions. Finite-precision identification stability and efficient ambiguity counting also warrant study.

%==============================================================================
\begin{ack}
The author thanks the anonymous reviewers and the area chair for constructive comments that improved the results and presentation. This work received no third-party funding. The author declares no competing interests.
\end{ack}

%==============================================================================
\bibliographystyle{plainnat}
\bibliography{references}

%==============================================================================
\appendix

\section{Proofs for Section~\ref{sec:framework}}
\label{app:framework}

We prove Lemmas~\ref{lem:counting}, \ref{lem:monotonicity}, and~\ref{lem:conditioning-simple} in their main-text order, using the answer objects and description lengths of Definitions~\ref{def:query-families} and~\ref{def:qcdl-general}.

\begin{proof}[Proof of Lemma~\ref{lem:counting}]
For any fixed $w$, there are at most $2^k - 1$ programs of length less than $k$ (summing over lengths $0, 1, \ldots, k-1$). Setting $k = \lfloor \log_2 N \rfloor$ shows that fewer than $N$ strings can have conditional complexity $K(\cdot \mid w)$ below $\log_2 N - O(1)$, so at least one must satisfy $K(z_i \mid w) \geq \log_2 N - O(1)$. For part~(2), at most $2^{\log_2 N - c} = N \cdot 2^{-c}$ strings can have conditional complexity below $\log_2 N - c - O(1)$, so the remaining fraction $\geq 1 - 2^{-c}$ satisfies the bound.
\end{proof}

\begin{proof}[Proof of Lemma~\ref{lem:monotonicity}] Direct from Definition~\ref{def:query-families}: $\Obs(M)$ is a component of $\Int_1(M)$, and $\Int_1(M)$ is a marginal of $\CF_1(M)$. Both reductions are computable.\end{proof}

\begin{proof}[Proof of Lemma~\ref{lem:conditioning-simple}]
The upper bound $K(\Int_1 \mid \Obs, n) \leq K(\Int_1 \mid n) + O(1)$ holds by ignoring the auxiliary input. For the lower bound, concatenate a program of uniformly bounded length that outputs $\Obs(M)$ from $n$ with a program that outputs $\Int_1(M)$ from $(\Obs(M), n)$: this yields $K(\Int_1 \mid n) \leq K(\Int_1 \mid \Obs, n) + O(1)$.
\end{proof}

\section{Proofs for Section~\ref{sec:interventional}}
\label{app:interventional}

This appendix contains all full proofs for Section~\ref{sec:interventional}, in the order of the main-text results.

\subsection{Rooted Trees}
\label{app:rooted-tree}

We give the formal construction and full proofs underlying the
$\Theta(n \log n)$ rooted-tree separation. Theorem~\ref{thm:tree-separation} below
restates and proves the main-text Theorem~\ref{thm:tree-main}; together with
Corollary~\ref{cor:degree-optimality} it shows the bound is order-optimal within
finite-gate-schema SCMs of maximum indegree $d=O(1)$.

\paragraph{Intuition.}

Consider $n$ variables that are always perfectly correlated: either all $0$ or all $1$, with equal probability. A passive observer sees only this correlation. But \emph{how} does the correlation arise?

Imagine the variables arranged in a tree, with one root variable $X_r$ that is uniformly random, and all other variables copying their parent. The tree structure determines the \emph{direction of information flow}, but this structure is invisible to passive observation: all trees produce the same ``all equal'' distribution.

Interventions reveal the hidden structure. If we force $X_i = 0$, then:
\begin{itemize}
    \item All \emph{descendants} of $i$ in the tree become $0$ (they copy from ancestors, and $i$ is now $0$).
    \item All \emph{non-descendants} remain random (they don't depend on $i$).
\end{itemize}

The interventional distributions reveal the descendant structure, which determines the tree. Cayley's formula gives $n^{n-1}$ rooted labeled trees, all observationally identical but interventionally distinct. The counting argument below therefore gives a $\Theta(n\log n)$ gap in the worst case and with high probability under a uniform tree.

\paragraph{Formal construction.} For a rooted tree $T$ on $[n]$ with root $r$, we write $\pa_T(v)$ for the parent of $v$ in $T$ (defined for all $v \neq r$), and $\Desc_T(i)$ for the set of \emph{descendants} of node $i$ in $T$, including $i$ itself. That is, $j \in \Desc_T(i)$ if and only if $i$ lies on the unique path from $j$ to $r$.

\begin{definition}[Rooted-Tree SCM]
\label{def:tree-scm}
Let $T$ be a rooted labeled tree on vertex set $[n]$ with root $r$. Define the SCM $M_T$:
\begin{itemize}
    \item \textbf{Exogenous variables:} $U_r \sim \mathrm{Bernoulli}(1/2)$. All other $U_i$ are constant (say, $0$) and play no role.
    \item \textbf{Structural equations:}
    \begin{align*}
        X_r &:= U_r, \\
        X_v &:= X_{\pa_T(v)} \quad \text{for each } v \neq r,
    \end{align*}
    where $\pa_T(v)$ denotes the parent of $v$ in the rooted tree $T$.
\end{itemize}
\end{definition}

Only the root label and parent links vary across this family.

\begin{lemma}[All Trees Have the Same Observations]
\label{lem:tree-obs}
For every rooted tree $T$, the observational distribution of $M_T$ is:
\[
P^\star(x_1, \ldots, x_n) = \begin{cases}
1/2 & \text{if } x_1 = x_2 = \cdots = x_n = 0, \\
1/2 & \text{if } x_1 = x_2 = \cdots = x_n = 1, \\
0 & \text{otherwise}.
\end{cases}
\]
Consequently, $\DL_1(M_T) = K(P^\star \mid n) = O(1)$.
\end{lemma}

Passive observations are identical for every rooted tree.

\begin{proof}
Since $T$ is a connected tree and every non-root node copies its parent, any value assigned to the root propagates to all nodes. Specifically:
\begin{itemize}
    \item If $U_r = 0$: Then $X_r = 0$. For any node $v$ at distance $1$ from $r$, we have $X_v = X_{\pa_T(v)} = X_r = 0$. By induction on distance from $r$, all nodes have value $0$.
    \item If $U_r = 1$: By the same argument, all nodes have value $1$.
\end{itemize}
Since $U_r \sim \mathrm{Bernoulli}(1/2)$, the outcome is $0^n$ with probability $1/2$ and $1^n$ with probability $1/2$.

The distribution $P^\star$ can be specified by a constant-length program: ``output $1/2$ for $0^n$ and $1^n$, output $0$ otherwise.'' Thus $K(P^\star \mid n) = O(1)$.
\end{proof}

\begin{lemma}[Interventions Reveal Descendants]
\label{lem:tree-intervention}
For any node $i \in [n]$ and the intervention $\doop(X_i = 0)$:
\[
P_{M_T}^{\doop(X_i = 0)}(X_j = 0) = \begin{cases}
1 & \text{if } j \in \Desc_T(i), \\
1/2 & \text{if } j \notin \Desc_T(i),
\end{cases}
\]
where $\Desc_T(i)$ is the set of descendants of $i$ in $T$ (including $i$ itself).
\end{lemma}

An intervention identifies which variables descend from its target.

\begin{proof}
Under $\doop(X_i = 0)$, the equation for $X_i$ becomes $X_i := 0$. All other equations remain unchanged.

\textbf{Case 1: $j \in \Desc_T(i)$.} Every descendant of $i$ lies on a path from $i$ going away from the root. Along this path, each node copies its parent. Since $X_i = 0$ and descendants copy from ancestors (with $i$ being an ancestor), we have $X_j = 0$ deterministically.

\textbf{Case 2: $j \notin \Desc_T(i)$.} If $j$ is not a descendant of $i$, then the path from the root $r$ to $j$ does not pass through $i$. The value of $X_j$ depends only on nodes along this path, which are unaffected by the intervention on $X_i$. Thus $X_j$ still equals $X_r = U_r \sim \mathrm{Bernoulli}(1/2)$, so $P(X_j = 0) = 1/2$.
\end{proof}

\begin{lemma}[Descendant Sets Determine the Tree]
\label{lem:descendants-determine-tree}
The collection of descendant sets $\{\Desc_T(i) : i \in [n]\}$ uniquely determines the rooted tree $T$.
\end{lemma}

The descendant relation determines both the root and every parent edge.

\begin{proof}
We show how to reconstruct $T$ from the descendant sets.

\textbf{Finding the root:} The root $r$ is the unique node with $\Desc_T(r) = [n]$ (every node is a descendant of the root).

\textbf{Finding parent-child relationships:} For any non-root node $v$, its parent $\pa_T(v)$ is characterized as follows. Consider all nodes $u \neq v$ such that $v \in \Desc_T(u)$ (i.e., $u$ is an ancestor of $v$). Among these, $\pa_T(v)$ is the one with the \emph{smallest} descendant set.

To see why: along the unique path from $v$ to the root $r$, descendant sets strictly increase:
\[
\Desc_T(v) \subsetneq \Desc_T(\pa_T(v)) \subsetneq \Desc_T(\pa_T(\pa_T(v))) \subsetneq \cdots \subsetneq \Desc_T(r) = [n].
\]
The strict inclusions hold because each step toward the root adds at least the parent to the descendant set. Thus the parent has the smallest descendant set among all ancestors.

This procedure reconstructs all parent-child edges, hence the entire tree.
\end{proof}

\begin{theorem}[Rooted-Tree Separation: $\Theta(n \log n)$]
\label{thm:tree-separation}
Let $\calM_{\mathrm{tree}}^n$ be the family of rooted-tree SCMs $M_T$ of Definition~\ref{def:tree-scm}, over all rooted labeled trees $T$ on $[n]$. Then:
\begin{enumerate}
    \item The map $T \mapsto \Int_1(M_T)$ is injective on $\calM_{\mathrm{tree}}^n$.
    \item \textbf{Lower bound:} There exists $T$ with $\Delta_{2|1}(M_T) \geq (n-1)\log_2 n - O(1)$.
    \item \textbf{Upper bound:} For all $T$, $\DL_2(M_T) \leq (n-1)\log_2 n + O(\log n)$.
    \item \textbf{High probability:} For a uniformly random rooted labeled tree $T$ and every $c\geq0$,
    \[
    \Prob\bigl[\Delta_{2|1}(M_T) \geq (n-1)\log_2 n - c - O(1)\bigr] \geq 1 - 2^{-c}.
    \]
\end{enumerate}
Consequently, $\Delta_{2|1}(M_T) = \Theta(n \log n)$ for all but an exponentially small fraction of trees.
\end{theorem}

Uniformly selected trees typically require a description of order $n\log n$.

\begin{proof}
\textbf{Part 1 (Injectivity):} By Lemma~\ref{lem:tree-intervention}, from $\Int_1(M_T)$ we can read off:
\[
\Desc_T(i) = \{j \in [n] : P_{M_T}^{\doop(X_i = 0)}(X_j = 0) = 1\}
\]
for each $i$. By Lemma~\ref{lem:descendants-determine-tree}, the descendant sets determine $T$. Thus different trees yield different interventional families.

\textbf{Part 2 (Lower bound):} By Cayley's formula, there are $n^{n-1}$ rooted labeled trees on $[n]$. (Cayley's formula gives $n^{n-2}$ labeled unrooted trees on $[n]$; choosing a root from $n$ vertices gives $n^{n-1}$ rooted labeled trees.) By injectivity, there are $n^{n-1}$ distinct elements in $\{\Int_1(M_T) : T \in \calM_{\mathrm{tree}}^n\}$.

By the counting bound for Kolmogorov complexity: among $N$ distinct strings, at least one has complexity $\geq \log_2 N - O(1)$. Applying this with $N = n^{n-1}$:
\[
\max_T K(\Int_1(M_T) \mid n) \geq \log_2(n^{n-1}) - O(1) = (n-1)\log_2 n - O(1).
\]

By Lemma~\ref{lem:tree-obs}, $\DL_1(M_T) = O(1)$ for all $T$. By Lemma~\ref{lem:conditioning-simple}, $\Delta_{2|1}(M_T) = \DL_2(M_T) \pm O(1)$.

\textbf{Part 3 (Upper bound):} For $n\geq2$, encode the base-$n$ rank of the Pr\"ufer sequence in $\lceil(n-2)\log_2 n\rceil$ bits, plus $O(\log n)$ bits for the root. The case $n=1$ has one constant-complexity oracle. Total: $(n-1)\log_2 n + O(\log n)$ bits.

A constant-length program can then simulate $M_T$ and output $\Int_1(M_T)$. Thus $K(\Int_1(M_T) \mid n) \leq (n-1)\log_2 n + O(\log n)$.

\textbf{Part 4 (High probability):} By part~(2) of the counting bound (Lemma~\ref{lem:counting}), among the $n^{n-1}$ distinct strings $\{\Int_1(M_T)\}$, all but a $2^{-c}$ fraction satisfy $K(\Int_1(M_T) \mid n) \geq (n-1)\log_2 n - c - O(1)$. Transferring to $\Delta_{2|1}$ via Lemma~\ref{lem:conditioning-simple} incurs an additional $O(1)$ term.
\end{proof}

\subsection{Bipartite Graphs}
\label{app:bipartite-proofs}

Figure~\ref{fig:constructions} places the rooted-tree, bipartite, and masked-parity constructions side by side; Figure~\ref{fig:bipartite-construction} illustrates Definition~\ref{def:bipartite-scm}; Figure~\ref{fig:gap-comparison} compares the bipartite and rooted-tree rates. We then prove the decoding lemma and quadratic separation.

\begin{figure}[htbp]
\centering
\resizebox{\linewidth}{!}{%
\begin{tikzpicture}[
    every node/.style={circle, draw, minimum size=0.46cm, inner sep=0pt, font=\scriptsize},
    arr/.style={-{Stealth[length=1.6mm]}, semithick},
    hid/.style={-{Stealth[length=1.6mm]}, very thick, red!70!black},
    lbl/.style={draw=none, rectangle, font=\scriptsize, align=center, inner sep=1pt}
]
% (a) rooted tree
\begin{scope}[xshift=0cm]
    \node (r) [fill=gray!30] at (0,0) {$r$};
    \node (u) at (-0.8,-0.9) {$u$};
    \node (v) at (0.8,-0.9) {$v$};
    \node (w) at (0.8,-1.8) {$w$};
    \draw[arr] (r) -- (u); \draw[arr] (r) -- (v); \draw[arr] (v) -- (w);
    \node[lbl] at (0,-2.7) {(a) rooted tree\\ $\doop(X_v{=}0)$: $w \to 0$, $u$ stays random};
\end{scope}
% (b) bipartite
\begin{scope}[xshift=4.8cm]
    \node (r) [fill=gray!30] at (0,0) {$r$};
    \node (a1) at (-0.7,-0.9) {$a_1$};
    \node (a2) at (0.7,-0.9) {$a_2$};
    \node (b1) at (-0.7,-1.8) {$b_1$};
    \node (b2) at (0.7,-1.8) {$b_2$};
    \draw[arr] (r) -- (a1); \draw[arr] (r) -- (a2);
    \draw[arr, gray!60, dashed] (r) to[out=-150, in=120] (b1);
    \draw[arr, gray!60, dashed] (r) to[out=-30, in=60] (b2);
    \draw[hid] (a1) -- (b1); \draw[hid] (a1) -- (b2); \draw[hid] (a2) -- (b2);
    \node[lbl] at (0,-2.7) {(b) bipartite\\ $\doop(X_{a_1}{=}0)$: $b_1, b_2 \to 0$, reveals row $1$};
\end{scope}
% (c) masked parity
\begin{scope}[xshift=9.8cm]
    \node (x1) at (-1.0,0) {$X_1$};
    \node (x2) at (0,0) {$X_2$};
    \node (x3) at (1.0,0) {$X_3$};
    \node (y1) at (-1.0,-1.2) {$Y_1$};
    \node (y2) at (0,-1.2) {$Y_2$};
    \node (y3) at (1.0,-1.2) {$Y_3$};
    \node (v1) [minimum size=0.3cm, draw=gray] at (-1.0,-1.95) {\tiny $V_1$};
    \node (v2) [minimum size=0.3cm, draw=gray] at (0,-1.95) {\tiny $V_2$};
    \node (v3) [minimum size=0.3cm, draw=gray] at (1.0,-1.95) {\tiny $V_3$};
    \draw[hid] (x1) -- (y1); \draw[hid] (x2) -- (y1); \draw[hid] (x2) -- (y2); \draw[hid] (x3) -- (y3);
    \draw[arr, gray] (v1) -- (y1); \draw[arr, gray] (v2) -- (y2); \draw[arr, gray] (v3) -- (y3);
    \node[lbl] at (0,-2.7) {(c) masked parity: unclamped bits stay fair\\ $Y_j^{(X_i \leftarrow 0)} \oplus Y_j^{(X_i \leftarrow 1)} = \mathbf{1}[i \in S_j]$};
\end{scope}
\end{tikzpicture}}
\caption{\textbf{Three hidden-parameter constructions.} In (b) and (c), red edges mark the hidden structure. (a) Rooted tree: the root is a fair bit and every other node copies its parent, so observations see only ``all equal'' while single-node interventions reveal descendant sets. (b) Bipartite: each layer-$B$ node is the AND of the root and its hidden neighbors in layer $A$, so $\doop(X_{a_i}{=}0)$ reveals row $i$ of the adjacency matrix. (c) Masked parity: each $Y_j$ is the parity of a hidden set $S_j$ of inputs masked by a private fair bit $V_j$, so unclamped coordinates remain independent fair bits under hard interventions, while counterfactuals with the noise held fixed reveal the sets.}
\label{fig:constructions}
\end{figure}

\begin{figure}[htbp]
    \centering
    \begin{tikzpicture}[
        node distance=0.9cm,
        every node/.style={circle, draw, minimum size=0.6cm, font=\small},
        arr/.style={-{Stealth[length=2mm]}, thick}
    ]
    % Root
    \node (R) [fill=gray!30] {$r$};
    
    % Layer A
    \node (A1) [below left=1.3cm and 2.2cm of R] {$a_1$};
    \node (A2) [right=0.7cm of A1] {$a_2$};
    \node (A3) [right=0.7cm of A2] {$a_3$};
    
    % Layer B
    \node (B1) [below=1.6cm of A1] {$b_1$};
    \node (B2) [below=1.6cm of A2] {$b_2$};
    \node (B3) [below=1.6cm of A3] {$b_3$};
    
    % Root to A edges (solid)
    \draw[arr] (R) -- (A1);
    \draw[arr] (R) -- (A2);
    \draw[arr] (R) -- (A3);
    
    % Root to B edges (dashed, implicit)
    \draw[arr, gray!50, dashed] (R) to[out=-90, in=90] (B1);
    \draw[arr, gray!50, dashed] (R) to[out=-90, in=90] (B2);
    \draw[arr, gray!50, dashed] (R) to[out=-90, in=90] (B3);
    
    % Bipartite edges G (red, the hidden structure)
    \draw[arr, red!70!black, very thick] (A1) -- (B1);
    \draw[arr, red!70!black, very thick] (A1) -- (B2);
    \draw[arr, red!70!black, very thick] (A2) -- (B2);
    \draw[arr, red!70!black, very thick] (A3) -- (B3);
    
    % Annotations
    \node[draw=none, rectangle, right=2.8cm of R, align=left, font=\small] {
        \textbf{Observations:}\\
        All variables equal.\\
        $P(0^n) = P(1^n) = 1/2$\\[0.2cm]
        $\DL_1(M_G) = O(1)$
    };
    
    \node[draw=none, rectangle, right=2.8cm of B2, align=left, font=\small] {
        \textbf{Intervention $\doop(X_{a_1} = 0)$:}\\
        $b_1, b_2 \to 0$ (edges from $a_1$)\\
        $b_3 \to$ random (no edge)\\[0.2cm]
        Reveals $N_G(a_1) = \{b_1, b_2\}$
    };
    
    % Legend
    \node[draw=none, rectangle, anchor=north] at (0, -5.4) {\small Red edges = hidden graph $G$ (revealed by interventions)};
    
    \end{tikzpicture}
    \caption{\textbf{The Bipartite-Graph Construction.} Each layer-$A$ node copies the root $r$; each layer-$B$ node takes the AND of the root and its hidden layer-$A$ neighbors (red edges). Observationally, all $2^{m^2}$ choices of $G$ produce the same distribution. Interventions on layer $A$ reveal which layer-$B$ nodes are neighbors.}
    \label{fig:bipartite-construction}
\end{figure}

\begin{figure}[htbp]
    \centering
    \begin{tikzpicture}
        \begin{axis}[
            width=11cm, height=7cm,
            xlabel={System size $n$},
            ylabel={Description length or gap (bits)},
            xmin=2, xmax=30, ymin=0, ymax=250,
            legend pos=north west,
            legend style={font=\small},
            grid=major, grid style={dashed, gray!30}, thick
        ]
        \addplot[color=blue, line width=2pt, domain=2:30, samples=50] {3};
        \addlegendentry{$\DL_1$ (Observations) $= O(1)$}
        \addplot[color=orange, line width=1.5pt, dashed, domain=2:30, samples=50] {(x-1)*ln(x)/ln(2)};
        \addlegendentry{$\Delta_{2|1}$ (Trees) $= \Theta(n \log n)$}
        \addplot[color=red, line width=2pt, domain=3:30, samples=50] {((x-1)/2)^2};
        \addlegendentry{$\Delta_{2|1}$ (Bipartite) $= \Theta(n^2)$}
        \addplot[fill=red!15, draw=none, domain=3:30, samples=50] {((x-1)/2)^2} \closedcycle;
        \end{axis}
    \end{tikzpicture}
    \caption{\textbf{Illustrative leading-order gap rates.} The constant blue line and the tree and bipartite curves depict asymptotic rates, omitting encoding constants and lower-order terms. They are not computed Kolmogorov complexities of individual models; the gap rates hold in the worst case and with high probability under uniform hidden graphs.}
    \label{fig:gap-comparison}
\end{figure}

\begin{proof}[Proof of Lemma~\ref{lem:bipartite-obs}(i)]
We analyze the two cases for the root's value.

\textbf{Case $U_r = 0$:} Then $X_r = 0$. Since each $X_{a_i} = X_r$, we have $X_{a_i} = 0$ for all $i$. For each $b_j$, the equation is:
\[
X_{b_j} = X_r \wedge \bigwedge_{a_i : (a_i, b_j) \in G} X_{a_i} = 0 \wedge (\cdots) = 0.
\]
Thus all variables are $0$.

\textbf{Case $U_r = 1$:} Then $X_r = 1$ and $X_{a_i} = 1$ for all $i$. For each $b_j$:
\[
X_{b_j} = X_r \wedge \bigwedge_{a_i : (a_i, b_j) \in G} X_{a_i} = 1 \wedge 1 \wedge \cdots \wedge 1 = 1.
\]
(If $b_j$ has no neighbors, the empty AND is $1$, so $X_{b_j} = 1 \wedge 1 = 1$.)

Thus all variables are $1$.

Since $U_r \sim \mathrm{Bernoulli}(1/2)$, the outcome is $0^n$ or $1^n$ each with probability $1/2$. This is $P^\star$, independent of $G$.
\end{proof}

\begin{proof}[Proof of Lemma~\ref{lem:bipartite-intervention}(ii)]
Under $\doop(X_{a_i} = 0)$, the equation for $X_{a_i}$ becomes $X_{a_i} := 0$. The root $X_r = U_r$ remains random, and all other layer-$A$ nodes satisfy $X_{a_k} = X_r$ for $k \neq i$.

\textbf{Case $(a_i, b_j) \in G$:} The equation for $X_{b_j}$ includes $X_{a_i}$ as an AND input:
\[
X_{b_j} = X_r \wedge X_{a_i} \wedge \bigwedge_{a_k \neq a_i : (a_k, b_j) \in G} X_{a_k}.
\]
Since $X_{a_i} = 0$, the AND evaluates to $0$ regardless of other inputs. Thus $X_{b_j} = 0$ deterministically, so $P(X_{b_j} = 0) = 1$.

\textbf{Case $(a_i, b_j) \notin G$:} The equation for $X_{b_j}$ does not include $X_{a_i}$:
\[
X_{b_j} = X_r \wedge \bigwedge_{a_k : (a_k, b_j) \in G} X_{a_k}.
\]
All terms in this AND equal $X_r$ (since $X_{a_k} = X_r$ for $k \neq i$), so
\[
X_{b_j} = X_r \wedge X_r \wedge \cdots \wedge X_r = X_r.
\]
This also holds for an empty neighbor set, since the empty AND is $1$.

Thus $X_{b_j} = X_r \sim \mathrm{Bernoulli}(1/2)$, so $P(X_{b_j} = 0) = 1/2$.
\end{proof}

\begin{proof}[Proof of Lemma~\ref{lem:bipartite-injective}(iii)]
By Lemma~\ref{lem:bipartite-intervention}(ii), from $\Int_1(M_G)$ we can extract:
\[
N_G(a_i) = \{b_j \in B : P_{M_G}^{\doop(X_{a_i} = 0)}(X_{b_j} = 0) = 1\}
\]
for each $a_i \in A$. The neighborhoods $\{N_G(a_i)\}_{i=1}^m$ determine all edges of $G$:
\[
(a_i, b_j) \in G \iff b_j \in N_G(a_i).
\]
Thus $\Int_1(M_G)$ determines $G$ uniquely.
\end{proof}

\begin{proof}[Proof of Theorem~\ref{thm:quadratic-separation}]
\textbf{Lower bound:} There are $2^{m^2}$ bipartite graphs $G \subseteq A \times B$ (each of the $m^2$ potential edges is present or absent). By Lemma~\ref{lem:bipartite-injective}(iii), there are $2^{m^2}$ distinct interventional families. By Lemma~\ref{lem:bipartite-obs}(i), all share $\Obs(M_G) = P^\star$ with $\DL_1 = O(1)$.

By the counting bound: among $2^{m^2}$ distinct strings, at least one has Kolmogorov complexity $\geq m^2 - O(1)$. By Lemma~\ref{lem:conditioning-simple}, $\Delta_{2|1}(M_G) = \DL_2(M_G) \pm O(1)$.

The high-probability statement follows from the stronger counting bound: all but a $2^{-c}$ fraction have complexity $\geq m^2 - c - O(1)$.

\textbf{Upper bound:} The graph $G$ can be encoded as an $m \times m$ binary adjacency matrix, requiring $m^2$ bits. Given $G$ and $n$, a constant-length program simulates $M_G$ and outputs $\Int_1(M_G)$. Thus $K(\Int_1(M_G) \mid n) \leq m^2 + O(1)$.
\end{proof}

\subsection{Degree-Sensitive Optimality}
\label{app:degree-optimality}

\begin{proof}[Proof of Theorem~\ref{thm:degree-upper}]
Parent inputs follow the canonical order of Definition~\ref{def:finite-gate-class}, so no input permutation needs encoding. For the stated library memberships, take $\mathrm{copy}_0(u)=u$, $\mathrm{copy}_k(x,u)=x_1$ for $k\geq1$, and $\mathrm{parity}_k(x,u)=(\bigoplus_{i=1}^k x_i)\oplus u$, including $k=0$. An SCM in $\calM_{n,d}(\Gamma,\Pi)$ can be encoded by listing:
\begin{enumerate}
    \item a topological ordering of the $n$ variables, using $O(n \log n)$ bits;
    \item for each variable, its parent set, chosen from at most $\sum_{k=0}^{d} \binom{n-1}{k}$ possibilities;
    \item for each variable, a gate from the fixed library $\Gamma$ and a noise distribution from the fixed library $\Pi$, using $O(1)$ bits per variable.
\end{enumerate}
The entire SCM thus has a description of length
\[
    O(n \log n) + n \log_2\!\Bigl(\sum_{k=0}^{d} \binom{n-1}{k}\Bigr) + O_{\Gamma,\Pi}(n).
\]
Given this description and $n$, a fixed program can compute all single-node interventional distributions and output $\Int_1(M)$. Hence the same expression upper-bounds $K(\Int_1(M) \mid n) = \DL_2(M)$.

Finally,
\[
    \sum_{k=0}^{d} \binom{n-1}{k} \;\leq\; (d+1) \left(\frac{e(n-1)}{d}\right)^d
\]
for $1 \leq d \leq n-1$, yielding the stated asymptotic form. Since $\Delta_{2|1}(M) \leq \DL_2(M) + O(1)$, the same upper bound applies to the gap.
\end{proof}

\begin{proof}[Proof of Corollary~\ref{cor:degree-optimality}]
Substitute $d = O(1)$ and $d = \Theta(n)$ into Theorem~\ref{thm:degree-upper}; the rooted-tree and bipartite constructions lie in $\calM_{n,1}$ and $\calM_{n,m+1}$ under the gate-schema libraries identified after Definition~\ref{def:finite-gate-class}, realizing the rates via Theorems~\ref{thm:tree-main} and~\ref{thm:quadratic-separation}.
\end{proof}

\begin{proof}[Proof of Corollary~\ref{cor:int-all-rate}]
The encoding in the proof of Theorem~\ref{thm:degree-upper} describes the SCM $M$ itself: its topological order, parent sets, gate choices, and noise distributions. From this description and $n$, a fixed program simulates $M$ under any hard intervention $\doop(X_S = x_S)$ and jointly across all required worlds using the same noise realization, so it outputs $\Int_{\mathrm{all}}(M)$ and $\CF_1(M)$, and more generally $\Ans_Q(M)$ for any query class $Q$ whose answer is computable from $M$. Hence $\DL(Q; M)$ obeys the same bound. For the transfer inequality, a fixed program maps $\Int_{\mathrm{all}}(M)$ to $\Int_1(M)$ by selecting the single-node components, so $K(\Int_1 \mid \Obs, n) \leq K(\Int_{\mathrm{all}} \mid \Obs, n) + O(1)$. The bipartite statement combines the transferred lower bound of Theorem~\ref{thm:quadratic-separation} (with $c = m^2/2$) and the upper bound at $d = m + 1 = \Theta(n)$, which is $O(n^2)$; the rooted-tree statement combines the transferred lower bound of Theorem~\ref{thm:tree-main} and the upper bound at $d = 1$, which is $O(n \log n)$.
\end{proof}

\begin{remark}[Why a bounded library is necessary]
\label{rem:bounded-library}
The restriction to finite libraries $\Gamma$ and $\Pi$ cannot be dropped, even with $n$ and $d$ fixed. Take a bipartite SCM $M_G$ with an edge $(a_i, b_j) \in G$ and let the response of $X_{b_j}$ on the configuration $(X_r, X_{a_i}) = (1, 0)$, which never occurs observationally, be $\mathrm{Bernoulli}(q)$ for an arbitrary rational $q$ instead of $0$; call the result $M_{G,q}$. Its observational oracle is still $P^\star$, but under $\doop(X_{a_i} = 0)$ the probability of $X_{b_j}=1$ equals $q/2$, so the single-node interventional oracle reveals $q$ and the gap is at least $K(q) - O(1)$, which is unbounded over rational $q$. Hence no gap bound depending only on $n$ and $d$ exists over unrestricted SCM classes. For any finite SCM family, the log-ambiguity lower bound of Theorem~\ref{thm:ambiguity-bits} still applies; equality up to lower-order terms requires a matching encoding of the relevant ambiguity class, as provided by the explicit codes for our constructions.
\end{remark}

\subsection{Finite-Precision Robustness}
\label{app:finite-precision}

\begin{proof}[Proof of Theorem~\ref{thm:approx-quadratic}]
\textbf{Upper bound.} An exact $m^2$-bit encoding of $G$ as an $m \times m$ binary adjacency matrix lets a fixed program output $\Int_1(M_G)$ exactly, hence also $\varepsilon$-approximately for any $\varepsilon \geq 0$. So $K_\varepsilon(\Int_1(M_G) \mid \Obs(M_G), n) \leq m^2 + O(1)$.

\textbf{Lower bound.} Let $G \neq G'$ be two distinct bipartite graphs. They differ on some edge $(a_i, b_j)$; WLOG $(a_i, b_j) \in G$ and $(a_i, b_j) \notin G'$. By Lemma~\ref{lem:bipartite-intervention}(ii),
\[
    P_{M_G}^{\doop(X_{a_i} = 0)}(X_{b_j} = 0) = 1, \qquad P_{M_{G'}}^{\doop(X_{a_i} = 0)}(X_{b_j} = 0) = 1/2.
\]
Total variation distance does not increase under marginalization, so the full interventional distributions satisfy
\[
    \mathrm{TV}\!\bigl( P_{M_G}^{\doop(X_{a_i} = 0)},\; P_{M_{G'}}^{\doop(X_{a_i} = 0)} \bigr) \;\geq\; \tfrac{1}{2}.
\]
Hence $d_{\Int}(\Int_1(M_G), \Int_1(M_{G'})) \geq 1/2$. For $\varepsilon < 1/4$, the closed $\varepsilon$-balls around the $2^{m^2}$ answer objects are pairwise disjoint, so any $\varepsilon$-accurate description must still distinguish among $2^{m^2}$ possibilities. Applying the conditional counting bound (Lemma~\ref{lem:counting}, part~2) with the common observation $P^\star$ as side information yields the existence and high-probability lower bounds.
\end{proof}

\section{Proofs for Section~\ref{sec:counterfactual}}
\label{app:xor-proofs}

We give full proofs for Section~\ref{sec:counterfactual}, following the main-text order.

\subsection{The Modular-XOR Family}
\label{app:modular-xor}

We prove parts (i)--(iv) of Lemma~\ref{lem:xor-obs} in order, followed by Theorem~\ref{thm:cf-separation}.

\begin{proof}[Proof of Lemma~\ref{lem:xor-obs}(i)]
We show each module $(X_t, Y_t)$ is uniformly distributed on $\{0,1\}^2$.

\textbf{Case $s_t = 0$:} $X_t = U_{X_t}$ and $Y_t = U_{Y_t}$ are independent $\mathrm{Bernoulli}(1/2)$ variables.

\textbf{Case $s_t = 1$:} $X_t = U_{X_t} \sim \mathrm{Bernoulli}(1/2)$. For $Y_t = X_t \oplus U_{Y_t}$: for any fixed $x \in \{0,1\}$,
\[
\Prob(Y_t = y \mid X_t = x) = \Prob(U_{Y_t} = x \oplus y) = 1/2.
\]
Thus $Y_t$ is uniform and independent of $X_t$. (Formally: $\Prob(X_t = x, Y_t = y) = \Prob(X_t = x) \Prob(Y_t = y) = 1/4$.)

Since modules are independent and each is uniform on $\{0,1\}^2$, the full distribution is uniform on $\{0,1\}^{2m}$.
\end{proof}

\begin{proof}[Proof of Lemma~\ref{lem:xor-int}(ii)]
Consider any single-variable intervention $\doop(X_t = x)$ or $\doop(Y_t = y)$.

\textbf{Intervention $\doop(X_t = x)$:} The equation $X_t := x$ replaces whatever $X_t$ was. In module $t$:
\begin{itemize}
    \item If $s_t = 0$: $Y_t = U_{Y_t}$, uniform.
    \item If $s_t = 1$: $Y_t = x \oplus U_{Y_t}$, also uniform (since $U_{Y_t}$ is uniform).
\end{itemize}
Either way, $Y_t$ is uniform and independent of the intervention value. Other modules are unaffected.

\textbf{Intervention $\doop(Y_t = y)$:} The equation $Y_t := y$ replaces whatever $Y_t$ was. In both cases ($s_t = 0$ or $1$), $X_t = U_{X_t}$ remains uniform. Other modules unaffected.

Since all single-variable interventional distributions are the same across all $s$, the entire family $\Int_1(M_s)$ is independent of $s$.
\end{proof}

\begin{proof}[Proof of Lemma~\ref{lem:xor-all-int}(iii)]
It suffices to consider a single module $(X_t, Y_t)$, since modules are independent and interventions factor across modules. Fix an arbitrary intervention on any subset of $\{X_t, Y_t\}$.

If $Y_t$ is intervened on, then $Y_t$ is fixed by the intervention. The remaining variable $X_t$, if not intervened on, equals $U_{X_t}$ and is uniform in both the no-effect and XOR mechanisms.

If $Y_t$ is not intervened on but $X_t$ is set to $x$, then
\[
    Y_t = \begin{cases} U_{Y_t}, & s_t = 0,\\ x \oplus U_{Y_t}, & s_t = 1. \end{cases}
\]
In either case $Y_t$ is uniform Bernoulli and independent of all other modules.

Finally, if neither variable is intervened on, then for $s_t = 0$ we have $(X_t, Y_t) = (U_{X_t}, U_{Y_t})$, uniform on $\{0,1\}^2$; for $s_t = 1$ we have $(X_t, Y_t) = (U_{X_t}, U_{X_t} \oplus U_{Y_t})$, also uniform on $\{0,1\}^2$.

Thus every possible intervention produces the same distribution regardless of $s_t$. Taking the product over independent modules proves the claim.
\end{proof}

\begin{proof}[Proof of Lemma~\ref{lem:xor-cf}(iv)]
For each module $t$, consider the counterfactual query: ``What is $\Prob(Y_t^{(X_t \leftarrow 0)} = Y_t^{(X_t \leftarrow 1)})$?''

Let $Y_t^{(X_t \leftarrow b)}$ denote the value of $Y_t$ when we intervene with $\doop(X_t = b)$, evaluated on the same noise $U_{Y_t}$.

\textbf{Case $s_t = 0$:} $Y_t^{(X_t \leftarrow 0)} = U_{Y_t}$ and $Y_t^{(X_t \leftarrow 1)} = U_{Y_t}$. These are identical, so $\Prob(Y_t^{(X_t \leftarrow 0)} = Y_t^{(X_t \leftarrow 1)}) = 1$.

\textbf{Case $s_t = 1$:} $Y_t^{(X_t \leftarrow 0)} = 0 \oplus U_{Y_t} = U_{Y_t}$ and $Y_t^{(X_t \leftarrow 1)} = 1 \oplus U_{Y_t} = 1 - U_{Y_t}$. These are always different, so $\Prob(Y_t^{(X_t \leftarrow 0)} = Y_t^{(X_t \leftarrow 1)}) = 0$.

Thus the counterfactual query for module $t$ outputs $1$ if $s_t = 0$ and $0$ if $s_t = 1$. The $m$ queries together determine $s$.

Since $\CF_1(M_s)$ contains $P(X, X^{(t \leftarrow 0)}, X^{(t \leftarrow 1)})$ for all $t$, which includes the joint distribution of $Y_t^{(X_t \leftarrow 0)}$ and $Y_t^{(X_t \leftarrow 1)}$, it determines $s$.
\end{proof}

\begin{proof}[Proof of Theorem~\ref{thm:cf-separation}]
\textbf{Common interventional oracle.} By Lemma~\ref{lem:xor-all-int}(iii), $\Int_{\mathrm{all}}(M_s)$ is identical for all $s$. Its description is constant: every intervention fixes the intervened variables and leaves each non-intervened module uniformly distributed in the manner described in the proof of Lemma~\ref{lem:xor-all-int}(iii).

\textbf{Upper bound.} Given $s$ ($m$ bits) and $n$, a constant-length program simulates $M_s$ and outputs $\CF_1(M_s)$. Thus $K(\CF_1(M_s) \mid n) \leq m + O(1)$.

\textbf{Lower bound.} By Lemma~\ref{lem:xor-cf}(iv), the map $s \mapsto \CF_1(M_s)$ is injective. Hence there are $2^m$ distinct counterfactual answer objects consistent with the same full hard-do interventional oracle. The conditional counting bound (Lemma~\ref{lem:counting}, part~2), applied with $w = (\Int_{\mathrm{all}}(M_s), n)$, yields the high-probability claim, since this auxiliary input is common to all $s$.
\end{proof}

\subsection{Masked Parity}
\label{app:masked-parity}

We now prove Theorem~\ref{thm:parity-separation}. Fix $n = 2m$, $1 \le d \le m$, and a hidden tuple $S$ as in \S\ref{sec:masked-parity}.

\begin{lemma}[Hard interventions see fair coins]
\label{lem:parity-int}
Fix any hard intervention $\doop(X_T = x_T, Y_W = y_W)$ with $T, W \subseteq [m]$, including the empty intervention. Under $M_S$ the post-interventional law is the product of the point masses $\delta_{x_T}$ and $\delta_{y_W}$ with independent $\mathrm{Bernoulli}(1/2)$ laws for every non-intervened coordinate, and in particular does not depend on $S$.
\end{lemma}

The private masks keep every nonclamped output uniformly random.

\begin{proof}
Non-intervened inputs satisfy $X_i = U_i$, independent fair bits. For a non-intervened output $Y_j$, write $Y_j = Z_j \oplus V_j$ where $Z_j := \bigoplus_{i \in S_j} X_i$ is a function of $(U, x_T)$. Since $V_j$ is a fair bit independent of $(U, V_{-j})$, conditionally on any value of $(U, V_{-j})$ the bit $Y_j$ is uniform; hence $Y_j$ is uniform and independent of all other coordinates, intervened or not. Taking the product over the non-intervened outputs gives the stated law. The empty intervention yields the uniform law on $\{0,1\}^{2m}$.
\end{proof}

\begin{lemma}[Counterfactuals reveal the sets]
\label{lem:parity-cf}
For all $i, j \in [m]$, $Y_j^{(X_i \leftarrow 0)} \oplus Y_j^{(X_i \leftarrow 1)} = \mathbf{1}[i \in S_j]$ with probability one under $M_S$. Consequently the map $S \mapsto \CF_1(M_S)$ is injective.
\end{lemma}

Comparing two worlds with the same noise reveals each hidden parent membership.

\begin{proof}
Both counterfactual worlds use the same exogenous realization $(u, v)$. In the world $\doop(X_i = b)$, $Y_j = \bigl(\bigoplus_{i' \in S_j \setminus \{i\}} u_{i'}\bigr) \oplus b\,\mathbf{1}[i \in S_j] \oplus v_j$, so the two worlds differ exactly in the term $b\,\mathbf{1}[i \in S_j]$. The component of $\CF_1(M_S)$ for the variable $X_i$ is the joint law of $(X, X^{(X_i \leftarrow 0)}, X^{(X_i \leftarrow 1)})$ over the whole system, whose marginal on $(Y_j^{(X_i \leftarrow 0)}, Y_j^{(X_i \leftarrow 1)})$ is supported on the diagonal if $i \notin S_j$ and on the anti-diagonal if $i \in S_j$. Reading this off for all pairs $(i, j)$ recovers $S$.
\end{proof}

\begin{lemma}[Rate of $L_{m,d}$]
\label{lem:parity-rate}
For $m \ge 2$ and $1 \le d \le m$, $L_{m,d} = m \log_2 \sum_{k=0}^{d}\binom{m}{k} = \Theta\bigl(nd \log\tfrac{en}{d}\bigr)$ with $n = 2m$.
\end{lemma}

The number of possible hidden parent sets determines the degree-sensitive rate.

\begin{proof}
For $1 \le d \le m$ the standard bounds $(m/d)^d \le \binom{m}{d} \le \sum_{k=0}^{d}\binom{m}{k} \le (em/d)^d$ give $d \log_2(m/d) \le L_{m,d}/m \le d \log_2(em/d)$. The upper bound is $O(md \log(em/d)) = O(nd \log(en/d))$. For $d \le m/2$ we have $\log_2(m/d) \ge 1$, so $\log_2(em/d) \le (1 + \log_2 e)\log_2(m/d)$ and $L_{m,d} = \Omega(md \log(em/d)) = \Omega(nd \log(en/d))$, using $\log(en/d) \le 2\log(em/d)$. For $m/2 < d \le m$, $\sum_{k \le d}\binom{m}{k} \ge 2^{m-1}$ gives $L_{m,d} \ge m(m-1) \ge m^2/2$, while $nd \log(en/d) \le 2m \cdot m \log(2e) = O(m^2)$.
\end{proof}

\begin{proof}[Proof of Theorem~\ref{thm:parity-separation}]
\textbf{Common observational and interventional laws.} Lemma~\ref{lem:parity-int} with the empty intervention gives the uniform observational law, and for every hard intervention it gives a law that depends only on $(T, x_T, W, y_W)$ and $n$. A constant-length program therefore outputs $\Int_{\mathrm{all}}(M_S)$ from $n$, so $\DL(\Int_{\mathrm{all}}; M_S) = O(1)$.

\textbf{Upper bound.} The tuple $S$ is one of $N_{m,d}$ possibilities and can be specified by its index in a canonical enumeration, using $\lceil L_{m,d} \rceil$ bits, plus $O(\log n)$ bits for $d$ and a self-delimiting header. Given $S$, $d$, and $n$, a fixed program simulates $M_S$ in each pair of counterfactual worlds and outputs $\CF_1(M_S)$. Hence $\DL_3(M_S) \le L_{m,d} + O(\log n)$.

\textbf{Lower bound and optimality.} By Lemma~\ref{lem:parity-cf} the $N_{m,d}$ strings $\{\CF_1(M_S)\}$ are distinct. As shown above, $\Int_{\mathrm{all}}(M_S)$ is a fixed string computable from $n$, so conditioning on it changes conditional complexity by at most $O(1)$, by the argument of Lemma~\ref{lem:conditioning-simple}. The counting bound (Lemma~\ref{lem:counting}, part 2) with auxiliary input $(\Int_{\mathrm{all}}(M_S), n)$ then gives $K(\CF_1(M_S) \mid \Int_{\mathrm{all}}(M_S), n) \ge L_{m,d} - c - O(1)$ for all but a $2^{-c}$ fraction of $S$, which is the stated probability bound for uniform $S$. The rate statement is Lemma~\ref{lem:parity-rate}, and order-optimality follows from Corollary~\ref{cor:int-all-rate} applied to $\CF_1$ in the class $\calM_{n,d}(\{\mathrm{parity}\}, \{\mathrm{Bernoulli}(1/2)\})$. Since $nd\log(en/d)\geq n\log n$ for $1\leq d\leq m$, the additional $n\log n$ term in that upper bound is absorbed.
\end{proof}

\section{Proofs for Section~\ref{sec:general}}
\label{app:general}

This appendix contains all full proofs for Section~\ref{sec:general}, in the order of the main-text results.

\subsection{Ambiguity and Entropy}
\label{app:ambiguity-proofs}

\begin{proof}[Proof of Theorem~\ref{thm:ambiguity-bits}]
The set $\calF_{Q_2 \mid Q_1}(a; \calM)$ is by construction a set of distinct finite strings. Apply the counting bound (Lemma~\ref{lem:counting}) with auxiliary input $w = (a, n)$.
\end{proof}

\begin{proof}[Proof of Theorem~\ref{thm:ambiguity-entropy}]
Conditional on $L(\Theta) = \ell$, the map $\Theta \mapsto U(\Theta)$ is a bijection on the support; bijections preserve Shannon entropy. If the conditional parameter law is uniform on $N$ values, its entropy is $\log_2 N$.
\end{proof}

For an explicitly listed finite family with computable answer maps and a specified computable prior, the Shannon conditional entropy in Theorem~\ref{thm:ambiguity-entropy} can be approximated to arbitrary accuracy by finite summation. This calculation does not compute $K$; its cost depends on the representation of the family and prior.

\subsection{Learning-Theoretic Consequences}
\label{app:learning}

The following two bounds and mutual-information observation prove Corollary~\ref{cor:nfl-main} for the bipartite family $\calM_{\mathrm{bip}}^n$ (Definition~\ref{def:bipartite-scm}).

\begin{corollary}[No-free-lunch for observational learners]
\label{cor:no-free-lunch}
Let $G \sim \mathrm{Unif}(2^{A \times B})$ in the bipartite family $\calM_{\mathrm{bip}}^n$. Any learner $\calA$, with private randomness independent of $G$, that receives an i.i.d.\ sample of any size from $\Obs(M_G)$ and outputs a prediction $\widehat{\Int}_1$ satisfies
\[
\Prob[\widehat{\Int}_1 = \Int_1(M_G)] \;\leq\; 2^{-m^2}.
\]
\end{corollary}

Independent observational data cannot improve exact recovery over guessing a uniform graph.

\begin{proof}
All $M_G$ share the observational distribution $P^\star$, so any sample from $P^\star$ is statistically independent of $G$. The learner's output, a function of the sample and independent private randomness, is therefore independent of $G$. With $2^{m^2}$ equiprobable targets $\{\Int_1(M_G)\}$ (the map $G \mapsto \Int_1(M_G)$ is injective by Lemma~\ref{lem:bipartite-injective}(iii)) and output independent of the true target, the success probability is at most $2^{-m^2}$.
\end{proof}

\begin{corollary}[Per-query prediction error]
\label{cor:per-query-error}
Let $G \sim \mathrm{Unif}(2^{A \times B})$ in the bipartite family $\calM_{\mathrm{bip}}^n$. Let $\hat{p}$ depend only on an i.i.d.\ observational sample, the queried indices, and private randomness independent of $G$. For an independent uniformly random pair $(i,j) \in [m] \times [m]$, let $p_{i,j}(G) := P_{M_G}^{\doop(X_{a_i} = 0)}(X_{b_j} = 0)$. Then
\[
\E\bigl[|\hat{p}_{i,j} - p_{i,j}(G)|\bigr] \;\geq\; \tfrac{1}{4}.
\]
\end{corollary}

This lower bound applies to the stated row-intervention marginal queries.

\begin{proof}
By Lemma~\ref{lem:bipartite-intervention}(ii), $p_{i,j}(G) \in \{1/2, 1\}$ depending on whether $(a_i, b_j) \in G$. For $G \sim \mathrm{Unif}(2^{A \times B})$ the edge indicator is Bernoulli$(1/2)$, so $p_{i,j}(G)$ equals $1$ and $1/2$ each with probability $1/2$. The observational sample is independent of $G$, hence so is the predictor output $\hat{p}_{i,j}$. Condition also on the queried indices $(i,j)$ and on $\hat{p}_{i,j}=a$:
\[
\E\bigl[|\hat{p}_{i,j} - p_{i,j}(G)| \,\big|\, \hat{p}_{i,j} = a\bigr] \;=\; \tfrac{1}{2} |a - 1| + \tfrac{1}{2} |a - \tfrac{1}{2}| \;\geq\; \tfrac{1}{4},
\]
where the last step uses the triangle inequality $|a-1| + |a-\tfrac{1}{2}| \geq |1 - \tfrac{1}{2}| = \tfrac{1}{2}$. Averaging over $a$ proves the claim.
\end{proof}

\begin{remark}[Shannon Mutual Information is Zero]
\label{rem:mutual-info}
In the bipartite family $\calM_{\mathrm{bip}}^n$, a dataset $D \sim (P^\star)^N$ of $N$ i.i.d.\ observational samples is independent of the hidden graph $G$, since all mechanisms share the same observational distribution. Therefore $I(G; D) = 0$ in the Shannon sense for every sample size $N$. The lower bounds in Corollary~\ref{cor:no-free-lunch} and Corollary~\ref{cor:per-query-error} follow from this information-theoretic independence and do not rely on uncomputability of Kolmogorov complexity.
\end{remark}

\subsection{From Bits to Experiments}
\label{app:experiments}

We first make the objects of \S\ref{sec:experiments} precise.

\begin{definition}[Experiments and adaptive strategies]
\label{def:experiments}
Fix a finite family $\calM$ on $n$ variables, query classes $Q_1,Q_2$, and an attained lower-rung answer $a$; write $\calF := \calF_{Q_2 \mid Q_1}(a; \calM)$ for the ambiguity class of Definition~\ref{def:gen-ambiguity} and $N := |\calF|$. An \emph{experiment} is a permitted intervention setting $q$ (for us, a hard intervention $\doop(X_S = x_S)$) whose outcome is the exact population law $\Ans_q(M)$. A setting can require many fresh-unit samples to estimate; one setting is not one sampled unit. Let $R(q) := |\{\Ans_q(M) : M \in \calM,\, \Ans_{Q_1}(M) = a\}|$ be the number of distinct outcomes of $q$ across models with lower-rung answer $a$ and $R := \max_q R(q)$. An \emph{adaptive strategy of depth $t$} chooses settings from previous outcomes and stops after at most $t$ settings. The proposition requires correct target identification for every model with lower-rung answer $a$. Short paths can be padded by repeating a setting, so the depth bound applies to variable stopping times too.
\end{definition}

\begin{proof}[Proof of Proposition~\ref{prop:bits-to-experiments}]
\textbf{General bound.} The outcome sequence of a depth-$t$ strategy is its transcript. Each experiment is determined by earlier outcomes, so there are at most $R^t$ transcripts. Since the output is a function of the transcript, at most $R^t$ distinct targets can be produced. Correctness for every compatible model requires at least $N$ distinct targets.

\textbf{Bipartite family.} $N = 2^{m^2}$ by Lemma~\ref{lem:bipartite-injective}(iii). Use the vertex layers $A$ and $B$ of Definition~\ref{def:bipartite-scm}. Fix a hard intervention $\doop(X_T = x_T)$ with $T \subseteq \{r\} \cup A \cup B$. We show that the post-interventional law of $M_G$ depends on $G$ only through the vector $z(G) \in \{0,1\}^{B \setminus T}$ with $z_j(G) := \mathbf{1}[\text{some } a_i \in N_G(b_j) \cap T \text{ has } x_{a_i} = 0]$, so that $R(q) \le 2^{m}$. Intervened coordinates are constants.

If $r \in T$ with $x_r = 0$, every non-intervened $a_k$ equals $0$ and every non-intervened $b_j$ equals $0 \wedge (\cdots) = 0$, a law independent of $G$.

If $r \in T$ with $x_r = 1$, every non-intervened $a_k$ equals $1$, and a non-intervened $b_j$ equals $1 \wedge \bigwedge_{a_i \in N_G(b_j)} X_{a_i}$, which is $0$ if $z_j(G) = 1$ and $1$ otherwise.

If $r \notin T$, then $X_r = U_r$, every non-intervened $a_k$ equals $X_r$, and a non-intervened $b_j$ equals $X_r \wedge \bigwedge_{a_i \in N_G(b_j) \cap T} x_{a_i} \wedge \bigwedge_{a_i \in N_G(b_j) \setminus T} X_r$, which is $0$ if $z_j(G) = 1$ and $X_r$ otherwise.

In all cases the law is a function of $(T, x_T, z(G))$, so at most $2^{|B \setminus T|} \le 2^m$ distinct outcomes occur. The general bound gives $t \ge \lceil m^2 / m \rceil = m$. For sufficiency, Lemma~\ref{lem:bipartite-intervention}(ii) shows that the outcomes of $\doop(X_{a_i} = 0)$ for $i \in [m]$ determine the rows $N_G(a_i)$, hence $G$ and $\Int_1(M_G)$, without adaptivity.

\textbf{Prescribed finite precision.} Fix an encoder $\rho_q$ that depends only on the post-interventional law, and let the experimental outcome be $\rho_q(\Ans_q(M))$. Define $R_\rho$ as above for these encoded outcomes. Deterministic encoding cannot increase the number of outcomes, so $R_\rho\leq R$, and the same transcript bound holds whenever the encoded outcomes permit target identification. For the bipartite family, if each output encodes a distribution within total variation $\varepsilon<1/4$ of the law, row interventions remain sufficient: their $b_j$-marginals lie strictly below or above $3/4$ according as the true probability of $X_{b_j}=0$ is $1/2$ or $1$. Thus the exact $m$-setting count also holds with such a fixed accurate encoder. Arbitrary model-dependent rounding or nonseparating encodings do not justify this conclusion.

\textbf{Modular XOR and masked parity.} By Lemma~\ref{lem:xor-all-int}(iii) and Theorem~\ref{thm:parity-separation}, every hard intervention has the same outcome for every member of the family, so $R(q) = 1$ for all $q$, while $N = 2^m$ for modular XOR (Lemma~\ref{lem:xor-cf}(iv)) and $N = N_{m,d} \ge (m+1)^m \ge 2^m$ for masked parity (Lemma~\ref{lem:parity-cf}). The general bound applies. The impossibility also holds for adaptively chosen samples from fresh units: conditional on every past transcript, the next setting has the same outcome distribution in every model. Induction gives identical finite transcript laws; these also determine the same law on infinite transcripts. Thus even unlimited fresh-unit samples cannot resolve the counterfactual target.
\end{proof}

\section{Proofs for Section~\ref{sec:cht}}
\label{app:assumptions}

This appendix gives the full statements and proofs for Section~\ref{sec:cht}, in main-text order. Propositions~\ref{prop:localization} and~\ref{prop:known-graph} establish part 1 of Proposition~\ref{prop:assumptions}; Proposition~\ref{prop:faithful} establishes part 2; Definition~\ref{def:semi-markov} and Theorem~\ref{thm:confounding} give the setting and construction for part 3. Proposition~\ref{prop:gaussian} and Remark~\ref{rem:identifiable} support the subsequent Gaussian discussion. Here, a Markovian SCM means an SCM with independent local noises, as in Definition~\ref{def:scm}.

\subsection{Graph Knowledge}

\begin{proposition}[Localization of the gap]
\label{prop:localization}
In the rooted-tree and bipartite families, $K(\Int_1(M_G) \mid \Obs(M_G), n) = K(G \mid n) \pm O(1)$, where $G$ is the hidden tree or bipartite graph.
\end{proposition}

In these families the missing information is precisely the hidden graph description.

\begin{proof}[Proof of Proposition~\ref{prop:localization}]
Let $G$ range over rooted trees or bipartite graphs and $M_G$ be the corresponding SCM. A fixed program maps $G$ (with $n$) to $\Int_1(M_G)$ by simulation, so $K(\Int_1(M_G) \mid n) \le K(G \mid n) + O(1)$. Conversely, Lemmas~\ref{lem:tree-intervention} and~\ref{lem:descendants-determine-tree} for trees, and Lemma~\ref{lem:bipartite-injective}(iii) for bipartite graphs, give fixed programs that recover $G$ from $\Int_1(M_G)$, so $K(G \mid n) \le K(\Int_1(M_G) \mid n) + O(1)$. Since $\Obs(M_G) = P^\star$ has constant description length, Lemma~\ref{lem:conditioning-simple} replaces $K(\Int_1(M_G) \mid n)$ by $K(\Int_1(M_G) \mid \Obs(M_G), n)$ at a cost of $O(1)$.
\end{proof}

\begin{proposition}[A known graph under positivity closes the gap]
\label{prop:known-graph}
Let $M$ be a Markovian SCM (Definition~\ref{def:scm}) whose observational law $P_M$ is strictly positive on $\{0,1\}^n$, and let $G$ be its graph. Then $K(\Int_{\mathrm{all}}(M) \mid \Obs(M), G, n) = O(1)$. Without positivity the conclusion fails: the SCMs $M_{G,q}$ of Remark~\ref{rem:bounded-library} share their graph and observational law, yet $K(\Int_1(M_{G,q}) \mid \Obs, G, n)$ is unbounded over rational $q$.
\end{proposition}

Positivity exposes every local response needed for truncated factorization.

\begin{proof}[Proof of Proposition~\ref{prop:known-graph}]
For a Markovian SCM with graph $G$ and strictly positive observational law $P$, the truncated factorization~\citep{pearl2009causality} gives, for every $S \subseteq [n]$ and $x_S$,
\[
P^{\doop(X_S = x_S)}(x) \;=\; \mathbf{1}[x'_S = x_S] \prod_{i \notin S} P\bigl(x_i \mid x_{\pa_G(i)}\bigr),
\]
where $x'_S$ denotes the $S$-coordinates of $x$. Positivity makes every conditional $P(x_i \mid x_{\pa_G(i)})$ well defined and computable from the rational table of $P$. A fixed program therefore outputs $\Int_{\mathrm{all}}(M)$ from $(P, G, n)$, which is the claim. For the failure without positivity, fix a bipartite graph $G$ with an edge $(a_i, b_j)$ and consider the SCMs $M_{G,q}$ of Remark~\ref{rem:bounded-library} for rational $q \in [0,1]$: $X_{b_j}$ keeps its parents and gains an independent $\mathrm{Bernoulli}(q)$ noise that is consulted only on the configuration $(X_r, X_{a_i}) = (1,0)$. All $M_{G,q}$ are Markovian with the same graph and the same observational law $P^\star$, but $P^{\doop(X_{a_i} = 0)}_{M_{G,q}}(X_{b_j} = 1) = q/2$, so $\Int_1(M_{G,q})$ determines $q$ and $K(\Int_1(M_{G,q}) \mid \Obs, G, n) \ge K(q) - O(1)$ is unbounded over $q$.
\end{proof}

\subsection{Positivity and Faithfulness}

\begin{proposition}[Positivity and faithfulness bound the worst-case gap at $\Theta(n\log n)$]
\label{prop:faithful}
Consider Markovian SCMs on $n$ binary variables.
\begin{enumerate}
    \item \textbf{Upper bound.} If $M$ is Markovian with a strictly positive observational law $P$ that is faithful to its graph, then $K(\Int_{\mathrm{all}}(M) \mid \Obs(M), n) \leq \log_2 |\calE(P)| + O(\log n) \leq \log_2(n!) + O(\log n)$, where $\calE(P)$ is the Markov equivalence class of DAGs faithful to $P$.
    \item \textbf{Lower bound.} Fix a strictly positive rational distribution $P$ on $\{0,1\}^n$ with no conditional independences. For each total ordering $\pi$ of $[n]$ let $M_\pi$ be the Markovian SCM on the complete DAG oriented by $\pi$ whose mechanisms are the chain-rule conditionals of $P$. All $n!$ models generate $P$, and $P$ is faithful to each of their graphs. Outside a proper algebraic subset of the simplex of positive laws, whose complement contains rational points, the oracles $\Int_1(M_\pi)$ are pairwise distinct; for such $P$ and $\pi$ uniform over orderings, $\mathsf{H}(\Int_1(M_\pi) \mid P) = \log_2(n!)$, and for every $c \geq 0$, $\Pr[K(\Int_1(M_\pi) \mid P, n) \geq \log_2(n!) - c - O(1)] \geq 1 - 2^{-c}$.
\end{enumerate}
\end{proposition}

Faithfulness restricts the residual graph choices but does not generally identify an ordering.

\begin{proof}[Proof of Proposition~\ref{prop:faithful}]
\textbf{Part 1.} Given the rational table of $P$ and $n$, a fixed program computes every conditional independence relation of $P$ exactly, and under faithfulness these relations determine the Markov equivalence class $\calE(P)$ of the true graph, for instance by running the PC algorithm~\citep{spirtes2000causation} with exact independence oracles and enumerating the DAGs consistent with the resulting completed partially directed acyclic graph (CPDAG) in a canonical order. The true graph $G$ is one of them and can be specified by its index using $\lceil \log_2 |\calE(P)| \rceil$ bits; with an $O(\log n)$ self-delimiting header this is $\log_2 |\calE(P)| + O(\log n)$ bits. Given $P$ and $G$, Proposition~\ref{prop:known-graph} yields $\Int_{\mathrm{all}}(M)$ by truncated factorization, which uses positivity. Finally $|\calE(P)| \le n!$: all members of the class share a skeleton, and a topological ordering of a member orients every skeleton edge from the earlier to the later endpoint, so the sets of topological orderings of distinct members are disjoint. Assigning each member its lexicographically first ordering therefore injects the class into the $n!$ orderings.

\textbf{Part 2.} For a total ordering $\pi$, write $\prec_\pi i$ for the set of predecessors of $i$. The chain rule $P(x) = \prod_i P(x_i \mid x_{\prec_\pi i})$ defines a Markovian SCM $M_\pi$ on the complete DAG oriented by $\pi$ whose mechanisms are these conditionals; the conditionals are rational because $P$ is, and $M_\pi$ generates $P$. Explicitly, for each variable choose $U_i$ uniform on $\{0,\ldots,L_i-1\}$, with $L_i$ a common multiple of its conditional-probability denominators, and set $X_i=\mathbf{1}[U_i<L_iP(X_i=1\mid x_{\prec_\pi i})]$. These independent finite noises realize the required rational conditionals. Since $P$ has no conditional independences and a complete DAG implies none, $P$ is faithful to each of these graphs.

For distinctness, fix two orderings $\pi \ne \pi'$ and a pair $(i, j)$ with $i \prec_\pi j$ and $j \prec_{\pi'} i$. In $M_{\pi'}$ the variable $X_j$ is a non-descendant of $X_i$, so $P^{\doop(X_i = b)}_{M_{\pi'}}(X_j = 1) = P(X_j = 1)$ for $b \in \{0,1\}$. In $M_\pi$ the same quantity is a rational function of the entries of $P$ whose denominators are positive marginals. The condition that the two single-node oracles agree is therefore a finite system of polynomial equations in the entries of $P$ after clearing denominators, and its solution set $Z_{\pi,\pi'}$ is a proper algebraic subset of the simplex provided one strictly positive $P$ violates it.

Such a witness exists: let $X_i \sim \mathrm{Bernoulli}(1/2)$, $X_j = X_i \oplus \mathrm{Bernoulli}(\delta)$ for some $\delta \in (0, 1/2)$, and all other coordinates independent fair bits. This $P$ is strictly positive, $P(X_j = 1 \mid x_{\prec_\pi j}) = P(X_j = 1 \mid x_i)$ because $X_j$ is independent of the remaining predecessors given $X_i$, and so $P^{\doop(X_i = 1)}_{M_\pi}(X_j = 1) = 1 - \delta \ne 1/2 = P(X_j = 1)$.

The set $Z_0$ of laws with at least one conditional independence is likewise a proper algebraic subset of the simplex, since each independence is polynomial and proper: for $X_i\perp X_j\mid X_C$, make $X_i,X_j$ a correlated noisy-copy pair and all remaining coordinates independent fair bits. This positive law violates that independence. Every nontrivial set-valued independence implies such a singleton-pair independence, so finitely many proper sets suffice.

A finite union of proper algebraic subsets of the simplex is a proper algebraic subset and contains no open set; hence the complement of $Z_0 \cup \bigcup_{\pi \ne \pi'} Z_{\pi,\pi'}$ is open and dense in the simplex and contains rational points. For any rational $P$ in this complement, the $n!$ oracles $\Int_1(M_\pi)$ are pairwise distinct.

For such $P$, Theorem~\ref{thm:ambiguity-entropy} with $\Theta = \pi$ uniform gives $\mathsf{H}(\Int_1(M_\pi) \mid P) = \log_2(n!)$, and Lemma~\ref{lem:counting} with auxiliary input $(P, n)$ applied to the $n!$ distinct strings gives the Kolmogorov bound. The upper bound of part 1 applies to each $M_\pi$ with $|\calE(P)| = n!$, so the conditional-complexity gap is $\Theta(n\log n)$ in the worst case and with high probability under uniform orderings, while its Shannon conditional entropy is $\log_2(n!)$.
\end{proof}

\subsection{Confounding}

Confounding is represented by explicitly shared latent roots, extending Definition~\ref{def:scm}.

\begin{definition}[Semi-Markovian SCM and observed graph]
\label{def:semi-markov}
A \emph{semi-Markovian SCM} has an acyclic endogenous graph and mutually independent finite-valued latent roots with finite rational probability tables. Each latent root is an input to at most two endogenous equations; local noise $U_i$ collects the latent inputs to equation $i$. The \emph{observed graph} $\calG$ has the endogenous directed edges and $i\leftrightarrow j$ when equations $i,j$ share a latent root. Query answers remain finite rational tables, so the description-length and ambiguity results apply. Theorem~\ref{thm:degree-upper} concerns the independent-local-noise subclass.
\end{definition}

\begin{theorem}[Confounding gap given the observed graph]
\label{thm:confounding}
Let $m\geq1$ and $n=2m$. For a hidden string $s \in \{0,1\}^m$ let $M_s$ consist of $m$ independent modules $(X_t, Y_t)$ in which $X_t \sim \mathrm{Bernoulli}(1/2)$, the exogenous response type $R_t = (Y_t(0), Y_t(1)) \in \{0,1\}^2$ has the conditional law given $X_t$ shown in Table~\ref{tab:confounded}, and $Y_t := Y_t(X_t)$. Then:
\begin{enumerate}
    \item every $M_s$ has the same observed graph $\calG$ ($X_t \to Y_t$ and $X_t \leftrightarrow Y_t$ in each module) and the same observational law, uniform on $\{0,1\}^{2m}$;
    \item $P_{M_s}(Y_t = 1 \mid \doop(X_t = 1))$ equals $1/2$ if $s_t = 0$ and $3/4$ if $s_t = 1$, so $s \mapsto \Int_1(M_s)$ is injective;
    \item for $s$ uniform, $\mathsf{H}(\Int_1(M_s) \mid \Obs, \calG) = m$; for every $s$, $K(\Int_1(M_s) \mid \Obs, \calG, n) \leq m + O(1)$; and for every $c \geq 0$, $\Pr[K(\Int_1(M_s) \mid \Obs, \calG, n) \geq m - c - O(1)] \geq 1 - 2^{-c}$;
    \item the high-probability complexity lower bound of part 3 holds for the $\varepsilon$-approximate gap of Definition~\ref{def:approx-dl} for every $0\leq\varepsilon < 1/8$.
\end{enumerate}
\end{theorem}

A shared latent root can hide a mechanism bit even when its observed mixed graph is known.

\begin{table}[h]
\centering
\small
\begin{tabular}{@{}lcc@{}}
\toprule
 & $R_t \mid X_t = 0$ & $R_t \mid X_t = 1$ \\
\midrule
$s_t = 0$ & $\tfrac12(\delta_{00} + \delta_{11})$ & $\tfrac12(\delta_{01} + \delta_{10})$ \\
$s_t = 1$ & $\tfrac12(\delta_{11} + \delta_{01})$ & $\tfrac12(\delta_{01} + \delta_{00})$ \\
\bottomrule
\end{tabular}
\caption{Conditional law of the response type $R_t = (Y_t(0), Y_t(1))$ given $X_t$ in the confounded module of Theorem~\ref{thm:confounding}; $\delta_{ab}$ is the point mass at $(Y_t(0), Y_t(1)) = (a, b)$.}
\label{tab:confounded}
\end{table}

\begin{proof}[Proof of Theorem~\ref{thm:confounding}]
It suffices to analyze one module, since modules are independent and every quantity factorizes. Write $R = (Y(0), Y(1))$ and read the rows of Table~\ref{tab:confounded}.

\textbf{Part 1.} Use a latent root $L$ uniform on four values. Map these values to $(X,R)$ as $(0,00),(0,11),(1,01),(1,10)$ for $s=0$, and $(0,11),(0,01),(1,01),(1,00)$ for $s=1$. Set $X=X(L)$ and $Y=R(L)_X$. Both equations use $L$, and $Y$ uses $X$, giving the same observed bow $X\to Y$, $X\leftrightarrow Y$ in either case; independent module roots meet Definition~\ref{def:semi-markov}. For the observational law, $P(Y = 1 \mid X = 0) = P(Y(0) = 1 \mid X = 0) = P(R \in \{10, 11\} \mid X = 0)$ equals $1/2$ in both rows, and $P(Y = 1 \mid X = 1) = P(R \in \{01, 11\} \mid X = 1)$ equals $1/2$ in both rows; with $X$ a fair bit, $(X, Y)$ is uniform on $\{0,1\}^2$.

\textbf{Part 2.} Under $\doop(X = 1)$, $Y = Y(1)$ and the exogenous law is unchanged, so $P(Y = 1 \mid \doop(X = 1)) = P(R \in \{01, 11\})$, computed from the marginal of $R$, which is the average of the two columns. For $s = 0$ the marginal is uniform on the four types, giving $1/2$. For $s = 1$ the marginal puts mass $1/4$, $1/2$, $1/4$ on the types $11$, $01$, $00$, giving $3/4$. Hence $\Int_1(M_s)$ determines every $s_t$, and $s \mapsto \Int_1(M_s)$ is injective.

\textbf{Part 3.} The observational law is uniform on $\{0,1\}^{2m}$ and $\calG$ is the same graph for all $s$; both are computable from $n$ by constant-length programs, so conditioning on them changes conditional complexity by $O(1)$. Given $s$ ($m$ bits) and $n$, a fixed program outputs $\Int_1(M_s)$, which gives the upper bound. The $2^m$ strings $\Int_1(M_s)$ are distinct by part 2, so Lemma~\ref{lem:counting} gives the high-probability lower bound, and Theorem~\ref{thm:ambiguity-entropy} with $s$ uniform gives $\mathsf{H}(\Int_1(M_s) \mid \Obs, \calG) = m$.

\textbf{Part 4.} If $s \ne s'$ differ at module $t$, the laws $P^{\doop(X_t = 1)}_{M_s}$ and $P^{\doop(X_t = 1)}_{M_{s'}}$ have $Y_t$-marginals that differ by $1/4$, and total variation does not increase under marginalization, so $d_{\Int}(\Int_1(M_s), \Int_1(M_{s'})) \ge 1/4$. For $\varepsilon < 1/8$ the $\varepsilon$-balls around the $2^m$ oracles are pairwise disjoint, and the argument of Theorem~\ref{thm:approx-quadratic} applies verbatim.
\end{proof}

\subsection{Gaussian Models}

For the continuous extension, encode the zero-mean observational Gaussian law by its rational covariance $\Sigma$. Encode the full hard-do oracle by the finite tuple $(A_S,C_S)_{S\subseteq[n]}$ such that its law under $\doop(X_S=x_S)$ is $\mathcal{N}(A_Sx_S,C_S)$, including clamped coordinates and possibly singular $C_S$. These rational matrices specify all real-valued settings symbolically; they replace the binary probability-table encoding of \S\ref{sec:framework}. SEM abbreviates structural equation model.

\begin{proposition}[Faithful linear Gaussian orderings]
\label{prop:gaussian}
Let $\Sigma$ be a rational positive-definite $n \times n$ matrix such that $P = \mathcal{N}(0, \Sigma)$ has no conditional independences. For each total ordering $\pi$ let $M_\pi$ be the linear Gaussian SEM on the complete DAG oriented by $\pi$ obtained by regressing each variable on its predecessors. Then every $M_\pi$ has rational coefficients and error variances and generates $P$, and $P$ is faithful to each of their graphs. Outside a proper algebraic subset of the cone of positive-definite matrices, whose complement contains rational points, the full hard-do oracles $\Int_{\mathrm{all}}(M_\pi)$ are pairwise distinct. For such $\Sigma$ and $\pi$ uniform over orderings, $\mathsf{H}(\Int_{\mathrm{all}}(M_\pi) \mid \Sigma) = \log_2(n!)$, for every $c \ge 0$, $\Pr[K(\Int_{\mathrm{all}}(M_\pi) \mid \Sigma, n) \ge \log_2(n!) - c - O(1)] \ge 1 - 2^{-c}$, and for every $\pi$, $K(\Int_{\mathrm{all}}(M_\pi) \mid \Sigma, n) \le \log_2(n!) + O(\log n)$.
\end{proposition}

Generic covariance matrices leave an ordering ambiguity of $n!$ distinct causal oracles.

\begin{proof}
Regressing $X_j$ on its predecessors $X_{\prec_\pi j}$ gives the coefficient vector $\Sigma_{j, \prec j}\Sigma_{\prec j, \prec j}^{-1}$ and residual variance $\Sigma_{jj} - \Sigma_{j,\prec j}\Sigma_{\prec j,\prec j}^{-1}\Sigma_{\prec j, j}$, which are rational because $\Sigma$ is; equivalently, $M_\pi$ is the $LDL^{\top}$ factorization of $\Sigma$ in the order $\pi$. The SEM $X_j = \beta_j^{\top} X_{\prec_\pi j} + \epsilon_j$ with independent Gaussian $\epsilon_j$ of these residual variances has covariance $\Sigma$, so it generates $P$, and faithfulness holds because $P$ has no conditional independences while a complete DAG implies none. Under $\doop(X_S = x_S)$ the law of $M_\pi$ is Gaussian with mean affine in $x_S$ and a covariance that depends only on $S$, both with rational coefficients, so $\Int_{\mathrm{all}}(M_\pi)$ is a finite rational object; given $\Sigma$, $n$, and the index of $\pi$ among the $n!$ orderings, a fixed program computes it, which gives the upper bound.

For distinctness, fix $\pi \ne \pi'$ and a pair $(i, j)$ with $i \prec_\pi j$ and $j \prec_{\pi'} i$. In $M_{\pi'}$, $X_j$ is a non-descendant of $X_i$, so $\E[X_j \mid \doop(X_i = x)] = 0$ for all $x$. In $M_\pi$, $\E[X_j \mid \doop(X_i = x)] = \tau_{ij}(\pi)\, x$, where the total effect $\tau_{ij}(\pi)$ is a rational function of the entries of $\Sigma$. Agreement of the two oracles forces $\tau_{ij}(\pi) = 0$, a polynomial condition after clearing denominators, and its solution set is proper if one positive-definite $\Sigma$ violates it.

A witness: $X_i \sim \mathcal{N}(0,1)$, $X_j = X_i + \eta$ with $\eta \sim \mathcal{N}(0,1)$ independent, and all other coordinates independent standard Gaussians, so $\Sigma_{ii} = 1$, $\Sigma_{jj} = 2$, $\Sigma_{ij} = 1$, and all other off-diagonal entries vanish. In $M_\pi$ every variable other than $X_j$ has zero regression coefficients on its predecessors, and $X_j$ has coefficient $1$ on $X_i$ and $0$ elsewhere, so $\tau_{ij}(\pi) = 1 \ne 0$.

Each nontrivial Gaussian conditional independence is a proper polynomial condition: a covariance with a correlated pair $i,j$ and all other coordinates independent violates $X_i\perp X_j\mid X_C$. The finite union of these conditions is therefore a proper algebraic subset.

A finite union of proper algebraic subsets of the open cone contains no open set, so its complement is open and dense and contains rational positive-definite matrices. For such $\Sigma$ the $n!$ oracles are pairwise distinct, and Theorem~\ref{thm:ambiguity-entropy} and Lemma~\ref{lem:counting} with auxiliary input $(\Sigma, n)$ give the entropy and Kolmogorov statements.
\end{proof}

\begin{samepage}
\begin{remark}[Identifiable subclasses collapse the exact gap]
Fix a class $\mathcal{C}$ with a finite observational encoding and a computable map from that encoding to its identified model and interventional oracle.\label{rem:identifiable} Then a fixed program gives $K(\Int_{\mathrm{all}}\mid\Obs,n)=O(1)$. Equal-error-variance linear Gaussian SEMs with rational covariance satisfy these conditions~\citep{peters2014identifiability}: search finitely many orderings for coincident sequential residual variances, using exact rational arithmetic; any successful ordering yields the unique weighted DAG by identifiability. LiNGAM gives another identifiability result~\citep{shimizu2006linear}, but its continuous non-Gaussian laws require a specified finite encoding and computable recovery before the same description-gap conclusion follows. Thus faithfulness alone permits a worst-case $\Theta(n\log n)$ Gaussian gap, whereas computable exact identification removes it. Approximate identification additionally needs separation or stability.
\end{remark}

\end{samepage}

A finite-precision extension could quantize the Gaussian parameters or compare laws under a specified metric. Quantitative bounds would require explicit regularity conditions, such as bounded coefficients and well-conditioned covariances on nonclamped coordinates, together with a specified range of intervention values.

% NeurIPS checklist omitted in the arXiv version.

\end{document}